\documentclass[preprint,12pt]{elsarticle}
\usepackage{amsmath, amsxtra, amsfonts, amssymb, amstext}
\usepackage{graphicx}
\usepackage{subcaption}
\usepackage{verbatim}

\usepackage{url}
\urldef{\mailsa}\path|{meierflo}@inf.ethz.ch|

\usepackage[utf8]{inputenc} 
\usepackage[T1]{fontenc}    
\usepackage{lmodern}
\usepackage{hyperref}       
\usepackage{url}            

\usepackage{tikz}
\usepackage{amsmath}
\usepackage{xspace}
\usepackage{thmtools}
\usepackage{mathtools}
\usepackage{booktabs}
\usepackage[algo2e,ruled,vlined,linesnumbered]{algorithm2e}
\usepackage{multirow}
\usepackage{bigdelim}
\usepackage{thmtools}
\usepackage{thm-restate}
\usepackage{pdfpages}
\usepackage{bm}
\usepackage{enumerate}
\usepackage{array}

\newtheorem{theorem}{Theorem}
\newtheorem{lemma}[theorem]{Lemma}
\newtheorem{corollary}[theorem]{Corollary}
\newdefinition{definition}[theorem]{Definition}
\newdefinition{remark}[theorem]{Remark}
\newdefinition{claim}[theorem]{Claim}
\newproof{proof}{Proof}

\newcommand{\assign}{\leftarrow}

\newcommand{\N}{\mathbb{N}}
\newcommand{\R}{\mathbb{R}}

\renewcommand{\epsilon}{\varepsilon}
\newcommand{\eps}{\varepsilon}

\DeclareMathOperator{\Exp}{E}
\DeclareMathOperator{\E}{E}
\DeclareMathOperator{\Var}{Var}

\definecolor{orange}{RGB}{255,127,0}

\definecolor{col}{RGB}{0,200,50}

\newcommand{\onemax}{\textsc{OneMax}\xspace}

\newcommand{\leadingones}{\textsc{Leading\-Ones}\xspace}

\journal{Theoretical Computer Science}

\begin{document}

\begin{frontmatter}
\title{The linear hidden subset problem for the (1+1)
EA with scheduled and adaptive mutation rates}

%
%
\author[eth]{Hafsteinn Einarsson}
\author[eth]{Marcelo Matheus Gauy}
\author[eth]{Johannes Lengler}
\author[eth]{Florian Meier\corref{cor1}\fnref{fn1}}
\ead{meierflo@inf.ethz.ch}
\author[eth]{Asier Mujika}
\author[eth]{Angelika Steger}
\author[eth]{Felix Weissenberger}
\address[eth]{Department of Computer Science, ETH Z\"urich\\
 Universit\"atstrasse 6, 8092 Z\"urich, Switzerland
}
\cortext[cor1]{Principal corresponding author}

%



%
%


\begin{abstract}
We study unbiased $(1+1)$ evolutionary algorithms on linear functions with an unknown number $n$ of bits with non-zero weight. Static algorithms achieve an optimal runtime  of $O(n (\ln  n)^{2+\epsilon})$, however, it remained unclear whether more dynamic parameter policies could yield better runtime guarantees.
We consider two setups: one where the mutation rate follows a fixed schedule, and one where it may be adapted depending on the history of the run. For the first setup,  we give a schedule that achieves a runtime of $(1\pm o(1))\beta  n \ln n$, where $\beta \approx 3.552$, which is an asymptotic improvement over the runtime of the static setup. Moreover, we show that no schedule admits a better runtime guarantee and that the optimal schedule is essentially unique. For the second setup, we show that the runtime can be further improved to $(1\pm o(1)) e n \ln n$, which matches the performance of algorithms that know $n$ in advance.

Finally, we study the related model of initial segment uncertainty with static position-dependent mutation rates, and derive asymptotically optimal lower bounds. This answers a question by Doerr, Doerr, and K\"otzing.
\end{abstract}

\end{frontmatter}

\section{Introduction}
\label{sec:Intro}

Mutation-based evolutionary algorithms (EAs) aim to optimize a fitness function $f$ by alternately executing two phases. In the mutation phase, new search points are created by mutating the current search points, while in the selection phase, certain search points (usually the fittest ones) are selected. Then, the optimization process is continued with the selected search points. The most basic EA, the $(1+1)$ EA, keeps at any step only one search point $x  \in \{0,1\}^{\mathbb{N}}$. In the mutation phase,   an \emph{offspring} $y $ of the current search point $x$ is created  by flipping each bit independently with  probability $p$, called the \emph{mutation rate}. In the selection phase,   the fitness values $f(x)$ and $f(y)$ are compared and the better search point is selected.

 The mutation rate is a critical parameter for mutation-based evolutionary algorithms. For example, for linear pseudo-boolean fitness functions $f:\{0,1\}^n \to \R$, Witt has shown in \cite{witt2013tight} that the optimal static\footnote{A dynamic or adaptive choice of the mutation rate can be beneficial, but only in lower order terms, see~\cite{doerr2016optimal}.} mutation rate for the $(1+1)$ EA is $1/n$, which leads to a runtime (the number of function evaluations before a global optimum is found) of $(1\pm o(1))en\ln n$.\footnote{Throughout this paper, we use the Landau notation $o(.)$, $O(.)$, $\Theta(.)$, \ldots  with respect to $n\to \infty$, see for example~\cite{CormenAlgorithms}.}
Interestingly, for any other mutation rate $c/n$, where $c$ is a constant, Witt proved a strictly larger runtime of $(1\pm o(1))e^c/c \cdot n\ln n$. This runtime is worse by roughly a factor of $1/c$ if $c<1$, and it becomes exponentially worse as $c>1$ grows. Thus, finding the optimal mutation rate may not only be difficult but also paramount.

Crucially, even for a simple function like \onemax \footnote{The \onemax function assigns to a bit string $x$ the number of $1$ bits in the string.}, the optimal mutation rate $1/n$ can only be used if the problem size $n$ is known. However, consider the following \emph{hidden subset problem}: the search space is $\{0,1\}^N$, but only a small subset of $n \ll N$ positions are fitness-relevant. We call this hidden set the \emph{support} of the fitness function, and we study fitness functions that depend linearly on the supporting bits. In this case, since $n$ is unknown, the optimal mutation rate is also unknown. This problem was proposed by Cathabard, Lehre, and Yao~\cite{cathabard2011non} and has been studied by Doerr, Doerr, and K\"otzing~\cite{doerr2015solving,doerr2017unknown} in the case of \onemax and \leadingones \footnote{The \leadingones function assigns to a bit string $x$ the number of consecutive ones in the beginning of the string. } instead of linear functions. 

Situations in which the fitness is a function of a small hidden subset of parameters occurs naturally in many practical applications, particularly in the context of big data. For example, complex models like a biospheric model or a neural network may come with an immense number of parameters, and the choice of parameters (which is feasible with sufficient data) often leads to high-dimensional optimization problems. However, it often turns out in hindsight that only a small subset of parameters are relevant, which is exactly the situation captured by the hidden subset problem.

In the aforementioned work~\cite{cathabard2011non,doerr2015solving,doerr2017unknown}, the problems were analyzed for a \emph{static} choice of mutation rates (cf.~below). However, when faced with unknown problem characteristics, it is natural to consider more dynamic parameter handling, either scheduled or adaptive ones. In \cite{doerr2017unknown} it was speculated that dynamic parameter handling could improve the runtime compared to the static setup. In this paper, we quantify the gain or loss of either method. We restrict ourselves to mutation-based $(1+1)$ EAs with standard bit mutation\footnote{I.e., each offspring is generated by flipping each bit with the same probability, but this probability may vary from round to round.}, and we distinguish three different types of parameter handling.
\begin{enumerate}
\item In the \emph{static} setup, a probability distribution $\mathcal D$ is fixed before the algorithm starts, and in each round the mutation rate is drawn from $\mathcal D$.\footnote{In the classification of~\cite{eiben1999parameter}, this is classified as \emph{Parameter Tuning}.}
\item In the \emph{scheduled} setup, a sequence $\mathcal D_t$ of probability distributions is fixed before the algorithm starts. Then, in the $t$-th round of the algorithm the mutation rate is drawn from $\mathcal D_t$.\footnote{This is also known as \emph{Deterministic Parameter Control}~\cite{eiben1999parameter}. As pointed out in~\cite{eiben1999parameter}, this term may cause confusion as the mutation rate is not necessarily deterministic.} 
\item In the \emph{adaptive} setup, the mutation rate at time $t$ may be chosen depending on the history of the run up to time $t-1$.
\end{enumerate}

 \subsection{Previous work and our contribution}\label{sec:results}

As mentioned before, Witt has shown in \cite{witt2013tight} that for known $n$, the optimal mutation rate for any linear function is $1/n$, yielding a runtime of ~$(1\pm o(1))en\ln n$. Strictly speaking, Witt only considered static mutation rates.
 However, his proof is based on a drift argument, and he shows that for a suitable potential function, the drift towards the optimum is strongest for mutation rate $1/n$. Thus, his proof also shows that no \emph{adaptive} policy for the mutation rate can beat the runtime of~$(1\pm o(1))en\ln n$.
  Therefore, in our more difficult setting where $1/n$ is unknown, the bound $(1\pm o(1))en\ln n$ is also a lower bound on the runtime with any parameter handling policy. The question is thus: how much do we lose compared to this lower bound, depending on the parameter handling.\smallskip

\noindent \textbf{Static Mutation Rate.} The static setup has been studied (for \onemax and the non-linear \leadingones function) in~\cite{doerr2015solving,doerr2017unknown}. For \onemax, it turned out that even with the best static setup, the runtime is asymptotically slower if $n$ is unknown.
 More precisely, for any static setup  the runtime is at least $\Omega(n \ln^2 n)$~\cite{doerr2017unknown}, and  this bound is tight up to $\ln \ln n$ factors.\footnote{Actually, the statement is much more beautiful, and they do have matching upper and lower bounds, see the section on Initial Segment Uncertainty for details.}
  Since \onemax is the easiest linear function by~\cite{witt2012optimizing}\footnote{For mutation rates at most $1/2$, see Section~\ref{sec:onemax_easiest}. It may be further seen that larger mutation rates are  detrimental for large $n$.
  }, the lower bound holds for \emph{every} linear function.

\noindent \textbf{Scheduled Mutation Rate.} For the scheduled setup, we show that there is an asymptotic improvement of the runtime over the runtime in the static setup. Moreover, the runtime is only by the  factor $\beta/e \approx 1.307$ larger than in the case where $n$ is known. More precisely, we show that the scheduling policy $\mathcal D^{\text{opt}}_t$ that sets the mutation rate in the $t$-th step deterministically to  $\alpha \ln (t)/t$ for $\alpha \approx 1.545$  leads whp\footnote{\emph{With high probability} denotes with probability $1- o(1)$.} to a runtime of $(1\pm o(1))\beta  n\ln n$ for every linear function with support of size $n$.
This policy is optimal, that is, for every other schedule\footnote{Strictly speaking: every schedule that deviates from $\mathcal D^{\text{opt}}_t$ by at least a constant factor for a significant density of $t$'s, see Remark \ref{remark:uniqueness} for more details.
}, deterministic or randomized,  there are infinitely many $n$ such that the runtime on every linear function with support of size $n$ is whp at least $(1\pm o(1))\beta'  en\ln n$ for some $\beta' > \beta$.

\noindent \textbf{Adaptive Mutation Rate.} Finally, we show that there is no significant price for the unknown $n$ if adaptive schedules are used: there is an adaptive scheduling scheme that achieves whp runtime~$(1\pm o(1))en\ln n$, thus matching the lower bound from the setting with known $n$.\smallskip

There are two ways to interpret the results. Firstly, we may define the \emph{black-box complexity} (BBC) of a function with respect to unbiased $(1+1)$ EAs with the respective updating scheme as the best runtime achievable by algorithms of this kind. In this sense, the result in~\cite{doerr2017unknown} says that the BBC of linear functions for static mutation rate is $\Omega(n\ln^2 n)$, while we show that the BBC for scheduled and adaptive mutation rates is $(1\pm o(1))\beta n\ln n$ and $(1\pm o(1))en\ln n$, respectively.

Secondly, we may consider this result as an analogue to the \emph{price of anarchy}~\cite{koutsoupias1999worst,roughgarden2002bad} in game theory. We may define the \emph{price of non-adaptiveness} of a family $\mathcal F = (f_n)_{n\geq 1}$ of functions, where $f_n$ has support of size $n$, to be
\begin{align*}
\text{PoNA}(\mathcal F) := \limsup_{n\to \infty}\frac{\text{runtime of best scheduled algorithm on $f_n$}}{\text{runtime of best adaptive algorithm on $f_n$}}.
\end{align*}
Then we show in this paper that for every family $\mathcal F = (f_n)$ of linear functions, we have PoNA$(\mathcal F) =\beta/e \approx 1.307$. Note that this definition also makes the somewhat ambiguous concept of ``best'' algorithm precise. For the adaptive case, there is a single algorithm which achieves, up to lower order terms, for all $n$ simultaneously the optimal runtime, so it is clear that this algorithm is best. For the scheduled setup, this is not the case, so we define the ``best'' algorithm as the algorithm which minimizes the PoNA.

\paragraph{Initial Segment Uncertainty}~\label{sec:initial}
Doerr, Doerr, and K\"otzing showed in~\cite{doerr2017unknown} that for \leadingones there is an intimate connection between the (static) hidden-subset problem (HSP) considered in this paper, and the following problem with \emph{initial segment uncertainty} (ISU). The support is an initial segment $\{1,\ldots,n\}$ of unknown length $n$, and for each bit $i$ the algorithm may choose a probability $p_i$. In each round the offspring is generated by flipping the $i$-th bit with probability $p_i$. This ISU variant was historically the first to be studied and  was motivated in~\cite{cathabard2011non} by the study of finite state machines~\cite{lehre2014runtime}. In~\cite{doerr2017unknown} it was conjectured that there is also a connection between the ISU model and the HSP for other problems than \leadingones, specifically for \onemax.

It was proved in~\cite{doerr2017unknown} that for every monotonically decreasing, \emph{summable}\footnote{A sequence $(d_i)_{i\in \mathbb{N}}$ is \emph{summable} if $\sum_{i \in \mathbb{N}} |d_i| < \infty$, and it is \emph{non-summable} otherwise. For details on (non)-summable sequence see the comprehensive exposition in~\cite{doerr2017unknown}.} sequence $(d_i)_{i\geq 1}$ of positive reals  there is an algorithm in the ISU model with runtime $O(\ln (n)/d_n)$ on \onemax. As an open problem, the authors asked for matching lower bounds. In this paper we provide such bounds, in the following sense. For every non-summable monotonically decreasing sequence $(d_i)_{i\geq 1}$ of positive reals there is a constant $c>0$ such that every algorithm in the ISU model has runtime at least $c \ln (n)/d_n$ on \onemax. Interestingly, both the upper and lower bound in the ISU model match the upper and lower bound in the HSP, which were derived in~\cite{doerr2017unknown}. Although this result is less tight than the connection for \leadingones (where the distributions of runtimes exactly coincide with each other), this gives further indication for a fundamental connection between the ISU model and the HSP.

\section{Notation, Algorithmic Setup, Tools}

\subsection{Models of uncertainty}
We consider a large search space $\{0,1\}^N$. In contrast, the function $f$ to be optimized\footnote{Throughout the paper, we assume a minimization problem.} only depends on a small subspace. More precisely, there is a set  $I \subset \{1,\ldots, N\}$ and a function $\tilde{f}$ on $\{0,1\}^n$, where $n \coloneqq |I|$, such that $f(x)=\tilde{f}(x_{|I})$ for all $x \in \{0,1\}^N$. Here, $x_{|I}$ denotes the bit string consisting of the bits $x_i$ of $x$ with $i\in I$. The dimension $N$ of the search space does not affect the results in this paper. Therefore, we assume that the search space is $\{0,1\}^{\mathbb{N}}$. We call the positions $I \subset \mathbb{N}$ that $f$ depends on the \emph{relevant bits} (or \emph{support}) of $f$. To ease notation, we also use the symbol $f$ for $\tilde f$.

We consider two models of uncertainty.
In the \emph{unrestricted uncertainty} model, the set of relevant bits $I$ and the number of relevant bits $n$ are unknown.
In the \emph{initial segment problem}, the set $I$ is the initial segment $[n]\coloneqq \{1, \ldots, n \}$, and the number of relevant bits $n$ is unknown.



\subsection{Algorithmic setup}
The $(1+1)$ EA has the goal of finding a search point that minimizes a function $f$. First, it draws u.a.r. a search point $x  \in \{0,1\}^{\mathbb{N}}$. Then,   an \emph{offspring} $y $ of the current search point $x$ is created in every round by flipping each bit independently with  probability $p$. The parameter $p$ is usually called \emph{mutation strength}, \emph{mutation rate} or \emph{mutation parameter}, in this paper we stick with \emph{mutation rate}.  If $f(y)\leq f(x)$, then  the search point $x$ is replaced by $y$, otherwise $x$ stays the current search point; we say that $y$ was \emph{accepted} or \emph{rejected}, respectively.\footnote{In this work, we only consider elitist algorithms, that is, the algorithm accepts the offspring $y$ of $x$ if and only if $f(y)\leq f(x)$. This is a natural choice since the drift with respect to the \textsc{OneMax} function is maximized by elitist algorithms. Thus, the runtime on \textsc{OneMax} cannot be improved by non-elitist algorithms. We show that the runtime on linear functions matches the one for \textsc{OneMax} in the scheduled and adaptive setup. Therefore, non-elitist algorithms cannot improve the runtime on linear functions in these settings either.}

For compact descriptions of Algorithms \ref{alg:1+1_EA} to \ref{alg:adaptiveEA_upper_bound}, we define the operator $\textsc{Mutate}(x, p)$, which generates a mutation $y$ of $x$ by flipping each bit independently with probability $p$ (if $p$ is a sequence, then each bit $x_i$ is flipped with probability $p_i$). 
In the static setup (see Algorithm \ref{alg:1+1_EA}), for each time $t$ the mutation rate $p_t$ is drawn from a fixed probability distribution $\mathcal{D}$ over the interval $[0,1]$, which is identical for all $t$. 
In the scheduled setup, a sequence of such distributions $\mathcal{D}_t$ is fixed in advance, and the mutation rates $p_t$ at time $t$ is drawn from~$\mathcal{D}_t$, see Algorithm~\ref{alg:tdEA}.
In the adaptive setup, the distributions $\mathcal{D}_t$ can depend on the history of the process. However, we assume that the algorithm is comparison-based, i.e., 
 whenever the fitness values of the search point $x$ and offspring $y$ are compared, the algorithm receives from an oracle the information whether the offspring is accepted or not. Then, it may choose $\mathcal{D}_t$ depending on all bits received from the oracle before time $t$. We also note that  all considered versions of the $(1+1)$ EA are unbiased, i.e., the mutation operator is invariant under the automorphisms of the search space. For more background on comparison-based and unbiased algorithms, see~\cite{doerr2017onemax}.

Finally, for the ISU model, we consider  position dependent mutation rates $\vec{p}$, where an offspring $y$ of $x$ is created by flipping the bit  at position $i$ with probability $p_i$, see Algorithm \ref{alg:position_dependent}. The $p_i$ are fixed over time.

As in previous work, we consider the number of fitness evaluations as the  complexity measure. We define the \emph{runtime} (or \emph{optimization time}) as the number of $f$-evaluations until the search point with minimal  $f$-value is reached.

\subsection{Basic Notation}
We denote sequences $(p_{t})_{t \in \mathbb{N}}$ by $\vec{p}$. In this paper, we consider the \textsc{OneMax} function and the class of linear functions $f$ to be minimized. The \textsc{OneMax} function with support $I$ is defined by $f(x)= \sum_{i\in I} x_i$ for any $x\in \{0,1\}^{\mathbb{N}}$. A \emph{linear function} $f$ with support $I$ depends linearly on the bits in $I$, that is, $f(x)= \sum_{i\in I} w_i x_i$ for some $w_i \in \mathbb{R}$.
 Since $f(x)$ can be written as $ \sum_{i:w_i>0}w_ix_i +\sum_{i:w_i < 0} |w_i| (1- x_i) -  \sum_{i:w_i < 0} |w_i|$, without loss of generality we can assume that $w_i>0$ for all $i \in I$. Therefore,  our target search point  is the all $0$ string from now on.

 Further, we denote by $x^t$ the search point at time $t$.
We say bit $i$ \emph{flips} at time $t$ if $y_i$ is set to $1-x_i^t$ in the mutation step of the algorithm. We say that there is a \emph{single bit flip} in round $t$ if exactly one relevant bit flips, and there is a \emph{multi bit flip} if at least two relevant bits flip. Further, we say bit $i$ \emph{changes} at time $t$ if $x^t_i \neq x^{t-1}_i$, which happens if bit $i$  flips  at time $t$ and the offspring $y$ is accepted.
We say that bit $i$ is \emph{optimized} at time $t$ if $x_i^t=0$.

\begin{algorithm2e}%
	\textbf{Initialization:} Sample $x \in \{0,1\}^{\mathbb{N}}$ uniformly at random\;
	\textbf{Optimization:}
	\For{$t=1,2,3,\ldots$}{
		$p_t \sim \mathcal{D}$\,; \\
		$y \leftarrow \textsc{Mutate}(x, p_t)$\;
		\lIf{$f(y)\leq f(x)$}{$x \assign y$\,; //selection step}
	}
	\caption{ The static $(1+1)$ EA with mutation rate distribution $\mathcal{D}$  minimizing a pseudo-Boolean function  $f\colon\{0,1\}^{\mathbb{N}}\to\mathbb{R}$ }
	\label{alg:1+1_EA}
\end{algorithm2e}

\begin{algorithm2e}%
	\textbf{Initialization:} Sample $x \in \{0,1\}^{\mathbb{N}}$ uniformly at random\;
	\textbf{Optimization:}
	\For{$t=1,2,3,\ldots$}{
		$p_t \sim \mathcal{D}_t$\,; \\
		$y \leftarrow  \textsc{Mutate}(x, p_t)$\;
		\lIf{$f(y)\leq f(x)$}{$x \assign y$\,; //selection step}
	}
	\caption{The scheduled $(1+1)$ EA with mutation rates drawn from a sequence of probability distributions $(\mathcal{D}_t)_{t\in \mathbb{N}}$  minimizing a pseudo-Boolean function  $f\colon\{0,1\}^{\mathbb{N}}\to\mathbb{R}$ with a finite number of relevant bits $n$.}
	\label{alg:tdEA}
\end{algorithm2e}

\begin{algorithm2e}%
	\textbf{Initialization:} Sample $x \in \{0,1\}^{\mathbb{N}}$ uniformly at random\;
	\textbf{Optimization:}
	\For{$t=1,2,3,\ldots$}{
		$y \leftarrow  \textsc{Mutate}(x, \vec{p})$\;
		\lIf{$f(y)\leq f(x)$}{$x \assign y$\,; //selection step}
	}
	\caption{The $(1+1)$ EA with position  dependent mutation rates $\vec{p}$  minimizing a pseudo-Boolean function  $f\colon\{0,1\}^{\mathbb{N}}\to\mathbb{R}$ that depends only on the initial segment of length $n$.}
	\label{alg:position_dependent}
\end{algorithm2e}

\subsection{ONEMAX is the easiest linear function} \label{sec:onemax_easiest}
Let $A$ and $B$ be two random variables that take values in $\mathbb{N}$.  $A$ \emph{stochastically dominates} $B$ if $\Pr(A \geq i) \geq \Pr(B \geq i) $ holds for all $i \in \mathbb{N}$. Witt showed  the following theorem for mutation based EAs with arbitrary population size in \cite{witt2013tight}. Here, we state it slightly less general.
\begin{theorem}(Theorem 6.2 in \cite{witt2013tight})
Consider the static $(1+1)$ EA A with mutation rate $p\leq 1/2$.   Then, the optimization time of algorithm $A$ on any function with a unique global optimum stochastically dominates  the optimization time of algorithm $A$ on \textsc{OneMax}.
\end{theorem}
 In \cite{witt2013tight}, Witt proved this theorem by induction over time $t$. The proof requires that $p\leq 1/2$ for every time step $t$, but does not require that $p$ is fixed.
Thus, the theorem can be extended to  the  setup where the mutation rates $\vec{p}$ are scheduled or adaptively chosen.
Therefore, Witt's proof also implies the following theorem.
\begin{theorem}(Adaptation of Theorem 6.2 in \cite{witt2013tight})\label{thm:onemax_easiest_adaptive}
Let algorithm $A$ be a $(1+1)$ EA with scheduled or adaptively  chosen mutation rates
  $\vec{p}$  satisfying $p_t \leq 1/2$. Then, the optimization time of algorithm A on any function with unique global optimum  stochastically dominates the optimization time of algorithm A on \textsc{OneMax}.

\end{theorem}

\subsection{Tools}
\label{sec:tools}
In the proofs throughout  this paper, we regularly use the following lemmas.
First, we state how sums can be approximated by integrals, see for example Chapter~10 in \cite{apostol2007calculus}.
\begin{lemma}[Integral test]\label{lemma:int_approx_sum}
Let $f: [1, \infty] \rightarrow \mathbb{R}_{\geq 0}$ be a monotone  function. Then, for any integer $n\geq 1$, it holds 
\begin{equation*}
\int_{1}^n f(t)\, \mathrm{d}t \leq \sum_{t=1}^n f(t) \leq \int_{1}^n f(t)\, \mathrm{d}t  + \max\{f(n), f(1)\}.
\end{equation*}
\end{lemma}

Next, the following bounds on $1-x$ for small $x$ turn out to be very useful. They follow easily from Taylor expansion.

\begin{lemma}\label{lemma:e_approx}
Let  $0 \leq x \leq \frac{1}{2}$, then it holds  that $e^{-x}\geq 1-x \geq e^{-x-x^2}\geq e^{-2x}$.
\end{lemma}

Further, in order to prove concentration of random variables we often use the method of bounded differences (see Theorem~5.3 in \cite{dubhashi2009concentration}), which is often referred to as Azuma's inequality.
\begin{lemma}[Method of Bounded Differences]
\label{lem:azuma}
Let $d_1, \dots, d_n$ be a sequence of reals and let $f:=f (x_1 , \ldots , x_n )$ be a function that satisfies for all $1\le i\le n$ 
\begin{align*}
|f(a)-f(a')| &\leq d_i
\end{align*}
whenever $a$ and $a'$ differ in just the $i$-th coordinate. Let $X_1 , \ldots , X_n$ be independent random variables, let $X=f(X_1,\ldots, X_n)$ and let $S= \sum_{i=1}^n d_i^2$, then
\begin{eqnarray*}
\Pr[X > \mathbb{E}[X]+t] &\leq & e^{-\frac{t^2}{2S}} \quad \text{ and } \\
\Pr[X < \mathbb{E}[X]-t] &\leq & e^{-\frac{t^2}{2S}} \ .
\end{eqnarray*}
\end{lemma}

Finally, we state the lower bound multiplicative drift theorem.

\begin{theorem}[Multiplicative drift, lower bound~\cite{doerr2017bounding}]\label{thm:multiplicative-drift}
Let $(X_t)_{t\geq 0}$ be random variables describing a Markov process over a finite state space $S\subset \mathbb{R}_{+}$. Let $\kappa > 0$, $s_{\text{min}} \geq \sqrt{2}\kappa$ and let $T$ be the random variable denoting the earliest point in time $t\geq 0$ such that $X_t\leq s_{min}$. If there exists a positive real $\delta >0$ such that, for all $x > s_{min}$ and $t\geq 0$ with $\Pr[X_t=s]>0$ it holds that 
\begin{enumerate}
\item $|X_t- X_{t+1}| \leq \kappa$, and 
\item $\mathbb{E}[X_t-X_{t+1}| X_t=s ]\leq \delta s$,
\end{enumerate}
then, for alll $s_0 \in S$ with $\Pr[X_0=s_0] > 0$, 
\begin{align*}
\mathbb{E}[T|X_0=s_0]\geq \frac{1+\ln s_0 - \ln s_{min}}{2 \delta + \frac{\kappa^2}{s_{min}^2-\kappa^2}} \ . 
\end{align*}
\end{theorem}

\section{Scheduled setup}\label{sec:time_dependent_mutation_probs}
First we give some intuition on how the mutation rates $\vec{p}$ should be chosen. It turns out that nearly all time of the optimization process is spent to optimize the last $\epsilon n$ non optimized bits  (for some $\epsilon >0$). In the regime where only very few $1$-bits are left the probability that the number of $1$-bits decreases given a multi bit flip is much smaller than the same probability given a single bit flip. If there were only single bit flips, then the fitness would improve every time a $1$-bit is flipped, and therefore a coupon collector type argument would imply that $(1\pm o(1))n \ln n$ single bit flips are necessary to optimize a function with $n$ relevant bits. Assuming there are $n$ relevant bits, the probability of a single bit flip is maximized by $p=1/n$. Since $n$ is unknown, we need to solve the problem for all $n$ simultaneously. If we fix $p_t$, then round $t$ contributes substantially to optimizing functions $f$ that have support size $n= \Theta(p_t^{-1})$, because for these $n$ the probability of a single bit flip is $np_t(1-p_t)^{n-1}= \Theta(1)$. We wish to optimize functions $f$ with support size $n$ in time $T_n=\Theta(n\ln n)$.  Since there is more time to optimize functions with large support $n$, for small $t$'s the $p_t$ should contribute to solving functions with small support. More precisely, for any $n$ a significant number of $p_t$'s with $t \leq T_n$  needs to be chosen of order $\Theta(1/n)$. This suggests to choose $p_t= \Theta(\ln (t)/t)$. As we will see, the optimal choice for the hidden factor will be a constant $\alpha$, which we now define together with the constant $\beta$.


\begin{definition}[$\alpha$, $\beta$]\label{def:alpha_beta}
Let $\alpha$ be the unique solution of the equation
 \begin{align}\label{eq:def_alpha}
 \int_{0}^{1}\frac{\alpha^{1-\frac{1}{u}}}{u} \mathrm{d}u=1 \ ,
 \end{align}
and let $\beta= \alpha / \ln \alpha$. The numerical values are approximately  $\alpha \approx 1.54468$ and $\beta \approx 3.55248$.
\end{definition}

\begin{remark} \label{remark:alpha_beta}
 Since the left hand side of Equation \ref{eq:def_alpha} is monotone decreasing in $\alpha$ (mind $u<1$) it is easy to see that there is a unique solution to Equation \ref{eq:def_alpha}. Further, the variable transformation $z= \alpha/\ln \alpha \cdot u$ transforms the integral into $\int_{0}^{\alpha / \ln \alpha}\frac{\alpha}{z}e^{-\frac{\alpha}{z}} \mathrm{d}z$, so the $\alpha, \beta$ from Definition~\ref{def:alpha_beta} satisfy
\begin{align}\label{eq:rem_alpha_beta}
\int_{0}^{\beta}\frac{\alpha}{z}e^{-\frac{\alpha}{z}} \mathrm{d}z = 1. 
\end{align}
Define 
$h(a,b)= \int_{0}^{b}\frac{a}{z}e^{-\frac{a}{z}} \mathrm{d}z$ for $a, b \geq 0$. 
We claim that the constants $\alpha, \beta$ from Definition \ref{def:alpha_beta} satisfy $\frac{\partial h}{\partial a}(\alpha,\beta)=0$. Indeed, this follows from
\[
\frac{\partial h}{\partial a}(\alpha,\beta) = \int_{0}^{\beta}(\frac{1}{z}-\frac{\alpha}{z^2})e^{-\frac{\alpha}{z}} \mathrm{d}z \stackrel{\eqref{eq:rem_alpha_beta}}{=} \frac{1}{\alpha} - \left[e^{-\frac{\alpha}{z}}\right]_{z=0}^{\beta} = \frac{1}{\alpha} - e^{-\frac{\alpha}{\beta}} \, \stackrel{\beta = \frac{\alpha}{\ln \alpha}}{=} \, 0.
\]
Moreover, it is easy to see that for fixed $b=\beta$, the value $a=\alpha$ is the only solution of $\frac{\partial h}{\partial a}(\alpha,\beta)=0$. Since $h$ is a non-negative function with $h(a,\beta)=0$ and $h(a,\beta) \to 0$ for $a\to \infty$, this means that for fixed $b=\beta$, the value $a= \alpha$ is the unique global maximum of $h(a,\beta)$. On the other hand, $h(a,b)$ is obviously increasing in $b$. Thus, for every $\beta' < \beta$ there is a $\delta >0$ such that $h(a, \beta')\leq (1-\delta)$  for all $a \geq 0$. Further, for every $\alpha' \neq \alpha$ there is a $\beta' > \beta$ and  $\delta >0$ such that  $h(\alpha', \beta')\leq (1-\delta)$.
\end{remark}

Now we are ready to state matching upper and lower bounds on the optimization time of the $(1+1)$ EA with scheduled mutation rates on linear functions  (see Algorithm \ref{alg:tdEA}).

\begin{theorem}[Lower bound]
\label{thm:linear_functions_lower_bound}
Let $(\mathcal{D}_t)_{t\in \mathbb{N}}$ be any scheduling policy   and let $\beta$ be as in Definition \ref{def:alpha_beta}. For infinitely many $n$, the optimization time of the $(1+1)$ EA with scheduling policy $(\mathcal{D}_t)_{t\in \mathbb{N}}$ on any linear function with $n$ relevant bits is whp at least $(1 - o(1))\beta n \ln n $.
\end{theorem}

It turns out that there is an optimal deterministic scheduling policy. Define $\mathcal{D}^{\text{opt}}_t$ to be the distribution that sets $p_t= \alpha \ln (t)/t$ with probability $1$, where $\alpha$ is defined in Definition $\ref{def:alpha_beta}$.

\begin{theorem}[Upper Bound]\label{thm:linear_function_upper_bound}
Let $\beta$ be defined as in Definition~\ref{def:alpha_beta}. Then the optimization time of the $(1+1)$ EA with scheduled policy $\mathcal{D}^{\text{opt}}_t$ is whp at most  $ (1+ o(1))\beta n \ln n $ for  any linear function $f$ with $n$ relevant bits.
\end{theorem}

As mentioned in the introduction the lower bound can be strengthened in the sense that Theorem \ref{thm:linear_functions_lower_bound} holds for a subset of $\mathbb{N}$ with positive density, and that the scheduling policy $\mathcal{D}^{\text{opt}}$ is essentially unique.
In order to make this precise, define the following measure $\mu$ on  $\mathbb{N}$. For any $N \subset \mathbb{N}$ define $\mu(N)= \sum_{t \in N} \ln (t)/t$. The \emph{density} of a set $N$ is defined as $\liminf_{n \to \infty} \mu(N \cap [n])/\mu([n])$.  For example, if a set contains for all $n$ at least $\epsilon n$ elements of $[n , 2n]$ for some $\epsilon > 0$, then it has positive density with respect to $\mu$. The proof of Theorem~\ref{thm:linear_functions_lower_bound} also shows the following two remarks.
\begin{remark}\label{remark:positiv_density}
Theorem \ref{thm:linear_functions_lower_bound} holds for a subset $N \subset{\mathbb{N}}$ with positive density with respect to $\mu$.
\end{remark}
\begin{remark}[Uniqueness of $\mathcal{D}^{\text{opt}}$]\label{remark:uniqueness}
Assume that a policy $(\mathcal{D}_t)_{t\in \mathbb{N}}$ deviates from $\mathcal{D}^{\text{opt}}$ on a set $N$ with  positive density with respect to $\mu$, that is, for each $t \in N$ either $p_t \leq (1-\epsilon) \alpha \ln (t)/t$ or  $p_t \geq (1+\epsilon) \alpha \ln (t)/t$ holds. Then, there is a $\beta'>\beta$ such that for infinitely many $n$, the optimization time of the $(1+1)$ EA with scheduling policy $(\mathcal{D}_t)_{t\in \mathbb{N}}$ on any linear function with $n$ relevant bits is whp at least $(1 - o(1))\beta' n \ln n $.
\end{remark}

\subsection{Proof of Lower Bound}
\label{sec:lower_bounds}

As discussed in Section \ref{sec:onemax_easiest}, the optimization time of the scheduled (1+1) EA  (Algorithm \ref{alg:tdEA}) on any function with unique global optimum stochastically dominates the optimization time of Algorithm \ref{alg:tdEA} on \textsc{OneMax}, for any sequence $\vec{p}$ of mutation rates. Therefore, in order to prove Theorem \ref{thm:linear_functions_lower_bound}  it suffices to show the following lemma.

\begin{lemma}
\label{thm:onemax_lower_bound}
Let $(\mathcal{D}_t)_{t\in \mathbb{N}}$ be any scheduling policy   and let $\beta$ be as in Definition \ref{def:alpha_beta}. For infinitely many $n$, the optimization time of the $(1+1)$ EA with scheduling policy $(\mathcal{D}_t)_{t\in \mathbb{N}}$ on the \textsc{OneMax} function with $n$ relevant bits is whp at least $(1 - o(1))\beta n \ln n $.
\end{lemma}

Due to the symmetry of the \textsc{OneMax} function, the relevant bits can be permuted arbitrarily. Therefore, we can assume from now on that the offspring $y$ of $x$ is only accepted if it has strictly better fitness $f(y)<f(x)$.
For the remainder of this section we 
let $T'_n =\beta' n \ln n$ for some arbitrary $\beta' < \beta$. 

Before we can prove Lemma~\ref{thm:onemax_lower_bound}, we first need some preparations. The following lemma follows easily from concentration inequalities. 
\begin{lemma}\label{lemma:start_with_few_one_bits}
It holds whp that   the \textsc{OneMax} function with $n$ relevant bits is  not optimized at time $T'_n$ or there was a point in time  with $cn/\ln^2 n$  relevant non-optimized bits for some $1\leq c \leq 2$.
\end{lemma}
\begin{proof}
The proof consists of three steps. Firstly, the Chernoff-Hoeffding bounds imply that the initial search point has whp roughly the same amount of $1$ and $0$ bits. Secondly, the number of $1$ bits whp does not jump from above $2n/\ln^2n$ to below $n/\ln^2$ in one time step. Thirdly, an union bound argument concludes that whp this does not happen in any timestep. 
More precisely, in the beginning of Algorithm \ref{alg:tdEA}, the bit string is initialized randomly. 
 For the initial number of $0$-bits $X$, the Chernoff-Hoeffding inequality implies that 
\begin{align*}
\Pr\left( \frac{0.99 n}{2}\leq X \leq \frac{1.01 n}{2} \right) \geq 1- e^{-0.01^2 n/3} \ .
\end{align*} 
  Since we want to prove that the statement of the lemma  holds whp, it is legitimate to assume $\frac{0.99 n}{2}\leq X \leq \frac{1.01 n}{2}$.   
 As the number of $0$-bits never decreases, the current number of $0$-bits $X_t$ is always at least $\frac{0.99 n}{2}$. Denote by $A_t$ the event that  $\frac{0.99n}{2}\leq X_t \leq n-2n/\ln^2 n$. Denote by $Y_t$ the number of $0$-bits in the offspring of the current search point. We bound the probability that $Y_t$ is larger than $n- n/\ln^2 n$. 
 For this purpose, note that $\mathbb{E}[Y_t\mid X_t=x]= x(1-2p)+np$. If $x\geq n/2$ then this value is at most $x$, and otherwise it is at most  $n-x$. Therefore, $\mathbb{E}[Y_t\mid A_t] \leq n-2n/\ln^2 n$.
  Since every bit is flipped independently with probability $p$ and the outcome of a single bit has an effect of at most $1$ on $Y_t|A_t$, Lemma \ref{lem:azuma} can be applied to show concentration of $Y_t|A_t$. It holds 
\begin{align*}
 \Pr \left(Y_t \geq  n-\frac{n}{\ln^2 n} \ \middle|\ A_t \right)
& \leq  \Pr\left(Y_t \geq \mathbb{E}[Y_t\mid A_t]  +  \frac{n}{\ln^2 n} \ \middle|\ A_t \right)
\leq  e^{-\frac{n}{2\ln ^4 n}}
\end{align*}
By a union bound argument over the first $\beta n \ln n $ rounds, it follows that whp the number of $0$-bits will not jump from below $n-2n/\ln^2 n$ to above $n-n/\ln^2 n$ in one step during these rounds. Therefore,  there is whp a round with $c n/\ln^2 n$ $1$-bits for some  $1\leq c \leq 2$ or the \textsc{OneMax} function is not optimized within these rounds. \qed
\end{proof}
To show that a statement holds with high probability like in Lemma~\ref{thm:onemax_lower_bound}, we may assume that other events of high probability do occur. In particular, by Lemma~\ref{lemma:start_with_few_one_bits} we may assume from now on that the process starts with $\ell_0 \coloneqq cn/\ln^2 n$  relevant non-optimized bits for some $1\leq c \leq 2$. It will turn out that in this situation it is rather unlikely that the fitness improves by multi bit flips. Thus, the next lemma which bounds the number of  single bit flips constitutes the core of the proof of Lemma \ref{thm:onemax_lower_bound}.

\begin{lemma} \label{lemma:concentration_one_bit_flips}
Given $n$ relevant bits, denote by $Z_n$ the number of single bit flips until time $T'_n$.
There exists a $\delta >0$ such that for infinitely many $n$ it holds $\mathbb{E}[Z_n] \leq (1-\delta/3) n \ln n $. For each of these $n$ it holds with probability $1-o(n^{-3})$  that $Z_n \leq (1-\delta/6) n \ln n $.
\end{lemma}
\begin{proof}
Let us first give a brief  proof sketch. We would like to bound the number of single bit flips $Z_n$ by $(1-\delta/3)n\ln n$ for certain $n$'s.
 It is clear that this bound is not true for all $n$, for example, if $p_t= 1/m$ for all $t$, then $\Exp[Z_m] = \beta'/e \cdot m\ln m$, which is larger than $m\ln m$ if $\beta' > e$. 
In order to  show that the bound holds for infinitely many $n$,  we consider instead a weighted average  $B=\sum_n \rho(n) Z_n$  over many $n$. 
The choice of $\rho$ is delicate, but it turns out that $\rho(n)=1/n^{2}$ is the right choice.\footnote{For example, any function $\rho(n) = \Theta(1/n^{2})$ would give a non-tight result if the $\Theta$ hides a function that oscillates by at least a factor $1\pm \eps$ for a constant $\eps>0$. The optimal scaling $\rho(n)= 1/n^{2}$ can be found by variational methods, but once it is known (or guessed), the derivation of $\rho$ is no longer required for a proof.}

 In the technical part of the proof, we derive an upper bound on $B=\sum_n Z_n/n^2$. Note that $Z_n$ counts the single bit flips until time $T'_n$. To bound $B$, we first change the order of summation and then approximate the contribution of each $p_t$ to $B$ by an integral. Using variable transformations, it  turns out  that the contribution of $p_t$ can be bounded by $\int_{0}^{\beta'}\tfrac{\alpha}{z}e^{-\alpha/z} \mathrm{d}z$, which is smaller than $1-\delta$ if $\beta' < \beta$. Finally, the upper bound on $B$ implies that $\mathbb{E}[Z_n] \leq (1-\delta/3) n \ln n$ for  infinitely  many $n$. 
 
Let us now make the proof precise. For now, we assume that $\vec{p}$ is given deterministically, and we will comment later how our proof generalizes if it is drawn from distributions $\mathcal{D}_t$. 
For any given $n$ let $Y_{t}^n$ be the indicator random variable that is $1$ if a single bit flip happens at time $t$ and $0$ otherwise. For reasons that will become clear later, we assume from now on that $Z_n$ counts the single bit flips from time $10\beta' =O(1)$ to time $T'_n$. Note that there are only constantly many single bit flips from time $1$ to time $10\beta' $ and therefore we can neglect this smaller order term. Thus, let $Z_n= \sum_{t=10 \beta'}^{T'_n} Y_{t}^n$.

Recall that $T'_n = \beta' n \ln n$ for some $\beta' < \beta$. Thus, there is a $\delta >0 $ such that 
$ \sup_{\alpha\in [0,\infty)} \int_{0}^{\beta'}\frac{\alpha}{z}e^{-\frac{\alpha}{z}} \mathrm{d}z = 1-\delta$ (cf. Remark \ref{remark:alpha_beta}).
For notational convenience define  real numbers $\alpha'_t$ such that $p_t=\alpha'_t\ln t/t$ holds. 
Let $N$ be an integer large enough such that $(1- \delta  +  \beta '/\ln N)(1+ 2/\ln N) \leq 1-5\delta/6$ holds and  let $M$ be an integer larger than $N$.  
Next, let us define $B$, which was mentioned in the proof sketch, precisely. Define $A= \sum_{n=N}^M \ln n/n$ and define
\begin{align*}
B &\coloneqq  \mathbb{E}\Big[\frac{1}{A}\sum_{n=N}^M \frac{1}{n^2} Z_n\Big] 
=  \frac{1}{A}\sum_{n=N}^M   \sum_{t=10 \beta' }^{T'_n}\frac{1}{n^2} \mathbb{E}[Y_{t}^n ].
\end{align*}
Now we switch the order of summation. 
Note that for $t/\beta'\geq 10$, both $\ln (t/\beta') >0 $ and $\ln \ln (t/\beta') > 0$ are satisfied. Therefore, 
 $\beta' n \ln n \geq t$ implies that 
 \[
 n \ln n \geq \frac{t}{\beta'}- \frac{t/\beta' \ln \ln (t/\beta')}{\ln (t/\beta')}= \frac{t/\beta'}{\ln (t/\beta')} \ln \frac{t/\beta'}{\ln (t/\beta')}.
 \]
The last inequality is equivalent to $n \geq \frac{t/\beta'}{\ln (t/\beta')}$, since $n \ln n$ is monotone. Thus, $\{ (n,t)\mid  N\leq n\leq M, 10 \beta' \leq t\leq \beta ' n \ln n \} $ is a subset of $\{ (n,t)\mid  10 \beta' \leq  t\leq \beta ' M \ln M , M\geq n\geq \max\{N, \frac{t/ \beta'}{ \ln(t/\beta')}\} \}$. Define $\hat{N}= \max\{N, \frac{t/ \beta'}{ \ln(t/\beta')}\}$. Then,
\begin{align*}
B&\leq \frac{1}{A}\sum_{t=10\beta'}^{T'_M}\sum_{n=\hat{N}}^M   \frac{1}{n^2}  \mathbb{E}[Y_{t}^n ] \text{ . }
\end{align*}
Define $S_1= \{ t \mid 10 \beta' \leq t \leq T'_M, p_t \geq \frac{1}{\ln N} \}$ and $S_2= \{ t \mid  10 \beta' \leq t \leq T'_M, p_t < \frac{1}{\ln N} \}$. We partition above sum according to $S_1$ and $S_2$ into $B_1$ and $B_2$, respectively. It holds $\frac{\mathbb{E}[Y_{t}^n ]}{n^2}= \frac{p_t}{n} (1- p_t)^{n-1}\leq \frac{1}{n} e^{- p_t(n-1)} \leq (e^{- \frac{1}{\ln N}})^{n-1}$ for $t \in S_1$. Therefore, using the formula for the sum of geometric series, it follows
\begin{align*}
B_1&\coloneqq  \frac{1}{A}\sum_{t\in S_1}\left(\sum_{n=\hat{N}}^M   \frac{1}{n^2}  \mathbb{E}[Y_{t}^n ] \right)
\leq  \frac{1}{A}\sum_{t\in S_1}\left(\sum_{n=\hat{N}}^{\infty}   (e^{- \frac{1}{\ln N}})^{n-1} \right)\\
&\leq  \frac{1}{A}\sum_{t=10 \beta'}^{\infty}  \frac{ \left(e^{- \frac{1}{\ln N}}\right)^{ \frac{t/\beta '}{ \ln (t/\beta')}-1} }{1- e^{-\frac{1}{\ln N}}} \ .
\end{align*}
Let $s_t= (e^{- \frac{1}{\ln N}})^{ \frac{t/ \beta'}{ \ln (t / \beta')}-1} /\big(1- e^{-\frac{1}{\ln N}}\big)$. We bound the sum   $\sum_{t=10 \beta'}^{\infty}  s_t$ by some constant $c$. Let $t_0$ such that for all $t\geq t_0$ it holds $ (e^{- \frac{1}{\ln N}})^{ \frac{t/\beta'}{ \ln (t/\beta')}-1} < \frac{1}{2\sqrt{t}}e^{-\sqrt{t}}$. Clearly, $\sum_{t=10\beta'}^{t_0} s_t< c_1$ for some constant $c_1$. For the second part of the sum it holds: $\sum_{t=t_0}^{\infty} s_t \leq s_{t_0}+ \int _{t_0}^{\infty} \frac{1}{2 \sqrt{t}}e^{-\sqrt{t}}\ \mathrm{d} t = \int _{\sqrt{t_0}}^{\infty} e^{-x} \ \mathrm{d}x=c_2$. Therefore, the whole sum can be bounded by some constant $c=c_1+c_2$. Lemma~\ref{lemma:int_approx_sum} implies that $A \geq \int_{N}^M \frac{\ln n}{n} \ \mathrm{d}n = \frac{\ln^2 M- \ln^2 N}{2}$. It follows that 
\begin{align*}
B_1 &\leq \frac{c}{A} \leq \frac{2c}{\ln ^ 2 M - \ln ^2 N} \leq \frac{\delta}{6} ,
\end{align*}
where the last inequality holds if $M$ is large enough.

Now, we bound the second  term $B_2$  of $B$ from above.
It holds that $ \frac{\mathbb{E}[Y_{t}^n]}{n^2}= \frac{p_t}{n} (1-p_t)^{n-1} \leq  \frac{p_t}{n} (1-p_t)^{n}(1+2p_t) \leq \frac{p_t}{n} e^{-np_t}(1+2p_t) $ and therefore, 
\begin{align*}
B_2 &\coloneqq  \frac{1}{A}\sum_{t\in S_2}\left(\sum_{n=\hat{N}}^M   \frac{1}{n^2}  \mathbb{E}[Y_{t}^n ]\right) 
\leq  \frac{1}{A}\sum_{t\in S_2 }\left( \sum_{n=\frac{t}{\beta' \ln t}}^M \frac{p_t}{n}    e^{-np_t} (1+2p_t)\right)\\
&\leq  \frac{1}{A}\sum_{t=10 \beta'}^{T'_M}\left(\sum_{n=\frac{t}{\beta' \ln t}}^M    \frac{p_t}{n} e^{-np_t} \left(1+ \frac{2}{ \ln N}\right)\right)\\
\end{align*}
Above we used  that all the summands are positive and $p_t \leq \frac{1}{ \ln N}$. Using Lemma \ref{lemma:int_approx_sum} to bound the inner sum by an integral, we get

\begin{align*}
 \sum_{n=\frac{t}{\beta' \ln t}}^M    \frac{p_t}{n} e^{-np_t} 
& \leq  \int_{\frac{t}{\beta' \ln t}}^M \frac{p_t}{n}   e^{-np_t}   \, \mathrm{d}n  + f\left(\frac{t}{\beta' \ln t}\right) \ ,
\end{align*}

where $f(\frac{t}{\beta' \ln t})
=\frac{\beta'\ln t}{t}   p_t e^{-\frac{\alpha'_t}{\beta'}}
\leq \frac{\ln t}{t}   p_t\beta' 
\leq \frac{\ln t}{t} \frac{\beta'}{\ln N} $.  

The integral can be rewritten using the variable transformation $x=\frac{t}{n \ln t}$ (implying $n=\frac{t}{x \ln t}$ and  $\frac{dn}{dx}= -\frac{ t}{x^2\ln t}$). First plugging in $p_t= \frac{\alpha'_t \ln t}{t}$ yields
\begin{align*}
&\int_{\frac{t}{\beta' \ln t}}^M \frac{1}{n}   \frac{\alpha'_t\ln t}{t} e^{-n\frac{\alpha'_t\ln t}{t}}   \, \mathrm{d}n
= \int_{\beta'}^{\frac{t}{M \ln t}} \left(\frac{x \alpha'_t \ln^2 t}{t^2}    e^{-\frac{\alpha'_t}{x}}\right) \left(-\frac{t}{x^2\ln t}\right)\, \mathrm{d}x\\
&\leq  \frac{ \ln t}{t} \int_{0}^{\beta'}  \frac{\alpha'_t}{x}    e^{-\frac{\alpha'_t}{x}}\, \mathrm{d}x   
\leq  \frac{ \ln t}{t} (1-\delta ) \ ,
\end{align*}
where  $\frac{\alpha'_t}{x} e^{-\frac{\alpha'_t}{x}} >0$ for $x>0$ implies the first inequality, and the definition of $\delta$ the second one.

Combining these bounds shows that
\begin{align*}
B_2 & \leq  \frac{1}{A}\sum_{t=10\beta'}^{T'_M}\frac{ \ln t}{t} \left(1- \delta  +  \frac{ \beta '}{\ln N}\right)\left(1+ \frac{2}{\ln N}\right) 
 \leq  \frac{1}{A}\sum_{t=10\beta'}^{T'_M}\frac{ \ln t}{t} \left(1- \frac{5\delta}{6}\right) \\
& \leq  \frac{\ln^2 (\beta' M \ln M) }{\ln ^2 M - \ln ^2 N }\left(1-\frac{\delta}{2}\right)
 \leq   1-\frac{4\delta }{6} \ , 
\end{align*}
where the last inequality holds for $M$ large enough.

It follows that  $B=B_1+B_2 \leq 1- \frac{\delta}{2}$ for $M$ large enough, and therefore,   it holds  $\mathbb{E}[Z_n] < (1-\frac{\delta }{3}) n \ln n$ for infinitely many $n$. This proves the first statement in Lemma~\ref{lemma:concentration_one_bit_flips}. Before we prove the probability tail bound, we give some remarks.

\begin{remark}\label{remark:density_proof}
It can be immediately seen that $\mathbb{E}[Z_n] < (1-\frac{\delta }{3}) n \ln n$ holds for a subset of  $\mathbb{N}$ with positive density with respect to the measure $\mu$ defined above Remark \ref{remark:positiv_density}. \
\end{remark}
\begin{remark}\label{remark:uniqueness_proof}
If  $\alpha'_t$ deviates from  $\alpha$, then for small enough  $\beta'> \beta$,  there is a $\delta' >0$ such that $\int_{0}^{\beta'}  \frac{\alpha'_t}{x}    e^{-\frac{\alpha'_t}{x}}\, \mathrm{d}x <1-\delta'$. Therefore if the $\alpha'_t$'s deviate from $\alpha$ on a subset of $\mathbb{N}$ with positive density $\epsilon'$. Then, $B$ can be bounded by $1-\frac{\epsilon'\delta'}{2}$ which  implies Remark \ref{remark:uniqueness}.
\end{remark}
\begin{remark}\label{remark:probabilistic_schedule}
The bound $B$ still holds if the $p_t$ are drawn from distributions $\mathcal{D}_t$. Assume that $\vec{p}$ is drawn from $\mathcal{D}= \otimes_{t\in \mathbb{N}} \mathcal{D}_t$, and define $B$ in the same way. 
\begin{align*}
B &= \frac{1}{A}\sum_{n=N}^M   \sum_{t=10 \beta' }^{T'_n}\frac{1}{n^2} \mathbb{E}[Y_{t}^n ]
 = \int \frac{1}{A}\sum_{n=N}^M   \sum_{t=10 \beta' }^{T'_n}\frac{p_t}{n} (1-p_t)^{n-1}  \, \mathrm{d} \mathcal{D}(\vec{p})\\
&\leq  \int \left(1-\frac{\delta}{2}\right)  \, \mathrm{d} \mathcal{D}(\vec{p})
= 1-\frac{\delta}{2} \ , 
\end{align*} 
where the inequality follows from the fact that the derived upper bound $1-\frac{\delta}{2}$ holds for $M$ large enough uniformly for all $\vec{p}$.
\end{remark}

For the second statement in Lemma~\ref{lemma:concentration_one_bit_flips}, it is left to show that $Z_n < (1-\frac{\delta }{6}) n \ln n$ with probability $o(n^{-3})$. Let $n$ be such that $\mathbb{E}[Z_n] < (1-\frac{\delta }{3}) n \ln n$. We apply Azuma's inequality to show concentration of $Z_n$.
Since the distributions $\mathcal{D}_t$ are independent, $Z_n= \sum_{t=1}^{T'_n} Y^{n}_{t}$ is  a sum of independent Bernoulli variables. The outcome of each $Y^{n}_{t}$ influences $Z_n$ by at most $1$, therefore by Lemma \ref{lem:azuma} we obtain 
\begin{align*}
   &\Pr \left( Z_n \geq (1-\frac{\delta}{6}) n \ln n\right) 
\leq  \Pr\left( Z_n \geq \mathbb{E}[Z_n] + \frac{\delta}{6} n \ln n  \right)\\
&\leq  e^{-\frac{(\delta  n \ln n/6)^2}{2 T'_n}} =e^{-\frac{\delta ^2 n \ln n}{72 \beta'}}= o( n ^{-3}) \ .  
\end{align*} \qed 

\end{proof} 

Now, we are ready to show Lemma~\ref{thm:onemax_lower_bound}.

\newproof{pot1}{Proof of Lemma \ref{thm:onemax_lower_bound}}
\begin{pot1}
Let $n$ be such that the statement of Lemma~\ref{lemma:concentration_one_bit_flips} holds, that is, the event $A$ that $Z_n \leq (1-\delta/6) n \ln n $ holds with probability $1-o(n^{-3})$, and recall that we assume to start with 
$\ell_0 = cn/\ln^2 n$   relevant non-optimized bits, where $c$ is some constant between $1\leq c \leq 2$.

Denote by $W_i$ the (Bernoulli) indicator random variable that is $1$ if the $i$-th of the initial $\ell_0$ relevant $1$-bits is $1$ at time $T'_n$ and let $W= \sum_{i=1}^{\ell_0} W_i$ be the number of such bits. 
Furthermore, denote by $V_i$ the Bernoulli random variable that is equal to $W_i$ if $Z_n > (1-\delta/6) n \ln n $. If $Z_n \leq (1-\delta/6) n \ln n $, then assume that after time $T'_n$   the optimization process continues with additional $  (1-\delta/6) n \ln n - Z_n $ random  single bit flips (i.e. in every round the offspring $y$ is produced by flipping one random bit of  $x$, and $y$ is accepted if $f(y)< f(x)$).  In this case we set $V_i$ to be $1$ if the $i$-th bit is $1$ after these additional random single bit flips, and let $V_i$ be $0$ otherwise. 
The advantage of the variables $V_i$ is that conditioned on  event $A$ there are exactly $(1-\delta/6) n \ln n $ single bit flips, which will make calculations simpler than with the variables $W_i$.
Denote $V= \sum_{i=1}^{\ell_0} V_i$.  
Clearly, it holds $V_i \leq W_i$, and thus $V\leq W$. Therefore, it is enough to show that whp $V >0$ in order to imply $W>0$ whp, which proves Lemma \ref{thm:onemax_lower_bound}. We will show this with the second moment method.
We claim that  
$  \mathrm{Var} (V) = O(\mathbb{E}[V]^2/\ln n)$. Then,  Chebyshev's inequality implies that
\begin{align*}
\Pr(V=0) & \leq  \Pr\big( |V-\mathbb{E}[V]| \geq \mathbb{E}[V]\big) \leq  \frac{\Var[V]}{\mathbb{E}[V]^2}=  O\Big(\frac{1}{\ln n}\Big) \ .
\end{align*}

In order to prove $  \mathrm{Var} (V) = O(\mathbb{E}[V]^2/\ln n)$, we need some additional  notation.

Let $i\neq j$ be two arbitrary relevant bits. Let $B_t$ be the random variable that denotes the number of flipped relevant bits at time $t$, and let $C_t$ be the random variable that denotes the number of relevant $1$-bits at time $t$.  Define 
\begin{align*}
q_{k,\ell} & \coloneqq  \Pr(x^t_i=0\mid x^{t-1}_i=1,  B_t=k, C_{t-1}=\ell),\\
r_{k, \ell} & \coloneqq  \Pr(x^t_i=0 \vee x^t_j=0 \mid x^{t-1}_i= x^{t-1}_j=1, B_t=k, C_{t-1}=\ell)\text.\\
\end{align*}

Recall that we can assume that the offspring only gets accepted if it has strictly better fitness. Then, the following claim follows easily, we  first will finish the proof of the lemma before we will prove the claim.

\begin{claim}\label{lemma_qk_bounds}
It holds that $q_{0, \ell}=0$, $r_{0,\ell}=0$, $q_{1,\ell} =  1/n$ and $r_{1,\ell} =  2/n$. For $\ell \leq n/20$ and $k \geq2$ it holds that $q_{k,\ell}\leq 160 \ell/n^2$.
\end{claim}
Conditioning on the event $A$, there are exactly $(1-\delta/6) n \ln n$ single bit flips and at most $T'_n$ multi bit flips. Above claim states that, the probability that bit $i$ changes in a single bit flip is $1/n$ and the probability that it changes in a multi bit flip is at most $\frac{160 \ell_0}{n^2}\leq\frac{320}{n \ln ^2 n}$. Therefore,

\begin{align*}
 \Pr(V_i=1)  &\geq  \Pr (V_i=1 \mid A) \Pr(A)  \\
&\geq   \Big(1- \frac{1}{n}\Big)^{(1-\frac{\delta}{6}) n \ln n} \Big(1- \frac{320}{n \ln ^2 n}\Big)^{T'_n}\Pr(A)\\
&= \Big(1- \frac{1}{n}\Big)^{(1-\frac{\delta}{6}) n \ln n} \Big(1-O\Big(\frac{1}{\ln n}\Big)\Big)\\
\end{align*}
By Lemma \ref{lemma:e_approx} it follows that 
\begin{align*}
\mathbb{E}[V] = \ell_0 \Pr(V_i=1)
\geq  \frac{n^{\delta / 6}}{\ln^2 n}\Big(1-O\Big(\frac{1}{\ln n}\Big)\Big)\text.
\end{align*}

In order to bound $\mathrm{Var}(V)$, we first bound $\Pr(V_i=V_j=1)$ from above. Note that the probability that neither $i$ nor $j$ changes in a single bit flip is $(1-2/n)$.
\begin{align*}
 &\Pr(V_i=V_j=1) \\
&=  \Pr (V_i=1 \wedge V_j=1\mid A) \Pr(A) + \Pr( V_i=1 \wedge V_j=1\mid \bar{A}) \Pr(\bar{A}) \\
&\leq  \Big(1- \frac{2}{n}\Big)^{(1-\frac{\delta}{6}) n \ln n}  +  \Pr(\bar{A}) \leq  \Big(1- \frac{1}{n}\Big)^{2 (1-\frac{\delta}{6}) n \ln n} + o\Big(\frac{1}{n^3}\Big) 
\end{align*}
We use this and the lower bound on $\Pr(V_i=1)$ to bound $\mathrm{Var}(V)$.
\begin{align*}
 \mathrm{Var}(V)&=\sum_{i,j} \Pr(V_i=V_j =1)- \Pr(V_i=1)\Pr(V_j=1)\\
& \leq  \mathbb{E}[V]+ \sum_{i\neq j}\Pr(V_i=V_j=1)- \Pr(V_i=1)\Pr(V_j=1)\\
&\leq  \mathbb{E}[V]+ \sum_{i\neq j} \Big(1- \frac{1}{n}\Big)^{2 (1-\frac{\delta}{6}) n \ln n} O\Big(\frac{1}{\ln n}\Big)\\
&\leq  \mathbb{E}[V]+ \sum_{i\neq j} \mathbb{E}[V_i]\mathbb{E}[V_j] O\Big(\frac{1}{\ln n}\Big)= O\Big(\frac{\mathbb{E}[V]^2}{\ln n}\Big)\text.
\end{align*}

It is left to show the claim. $q_{0, \ell}=0$ and $r_{0,\ell}=0$ holds because the search point does not change if no bit flips. If exactly one bit flips, then the probability that a specific bit is flipped is $1/n$. Therefore, 
$q_{1,\ell}  = 1/n$ and $r_{1, \ell}  =  2/n$ hold.

For $k \geq 2 \ell$, it is clear that the fitness cannot improve and therefore $q_{k,\ell}=0$.

Now, let us consider $2 \leq k \leq n/10$. 

\begin{align*}
q_{k, \ell} &= \Pr(x^t_i=0 \mid x^{t-1}_i=1 ,  B_t=k, C_{t-1}=\ell)=    \frac{\sum_{i=\lceil \frac{k-1}{2} \rceil}^k { \ell -1 \choose i}{n- \ell  \choose k-i-1}}{{ n \choose k }}
\end{align*}
Using that ${ n \choose m +1 } = { n \choose m } (\frac{n-m}{m+1})$, we can bound the $(i+1)$-th summand in terms of the $i$-th summand. 
\begin{align*}
 & { \ell -1 \choose i+1}{n- \ell  \choose k-(i+1)-1} \\
&= { \ell -1 \choose i}{n- \ell  \choose k-i-1} \frac{(\ell -1-i)(k-i-1)}{(i+1)(n-\ell - (k-(i+1)-1))}\\
&\leq  { \ell -1 \choose i}{n- \ell  \choose k-i-1} \frac{(\ell -1-i)}{(n-\ell - (k-(i+1)-1))} \\
&\leq  { \ell -1 \choose i}{n- \ell  \choose k-i-1} \frac{\ell}{\frac{8}{10}n}\ ,
\end{align*} 
where the first inequality follows from $i \geq \lceil\frac{k-1}{2}\rceil$.
Now, we can bound $q_{k,\ell}$ by a geometric series. Note that $\frac{n^k}{k!} \geq {n \choose k}\geq \frac{(n-k)^k}{k!}$ will imply the third inequality and that we use $a\leq 2^{a-1}$ for $a\geq 2$ in the fifth inequality.
\begin{align*}
 q_{k, \ell} &\leq  \frac{{ \ell -1 \choose \lceil \frac{k-1}{2} \rceil}{n- \ell  \choose \lfloor \frac{k-1}{2} \rfloor} \frac{1}{1-\frac{10 \ell}{8n}}}{{n \choose k}}
 \leq  2\frac{{ \ell -1 \choose \lceil \frac{k-1}{2} \rceil}{n- \ell  \choose \lfloor \frac{k-1}{2} \rfloor} }{{n \choose k}} \\
&\leq  2\frac{\left(\lfloor \frac{k-1}{2}\rfloor + 1\right)\ell^{\lceil \frac{k-1}{2} \rceil} (n-\ell)^{\lfloor \frac{k-1}{2} \rfloor}}{(n-k)^k} {k \choose \lceil \frac{k-1}{2} \rceil} \\
& \leq  \frac{2}{n-k}\left(\frac{\ell}{n-k}\right)^{\lceil \frac{k-1}{2} \rceil}\left(\frac{n-\ell}{n-k}\right)^{\lfloor \frac{k-1}{2} \rfloor}2^k\left(\lfloor  \frac{k-1}{2} \rfloor + 1\right) \\ 
& \leq  \frac{8}{n}\left(\frac{\ell}{n-k}\right)^{\lceil \frac{k-1}{2} \rceil}2^{\lfloor \frac{k-1}{2} \rfloor}2^{k-1}2^{\lfloor \frac{k-1}{2} \rfloor} \\
& \leq  \frac{8}{n}\left(\frac{16\ell}{n-k}\right)^{\lceil \frac{k-1}{2} \rceil}  \leq  \frac{8}{n}\left(\frac{16\ell}{\frac{9}{10}n}\right)^{\lceil \frac{k-1}{2} \rceil}  \leq  \frac{160\ell}{n^2} 
\end{align*}
\end{pot1}

\subsection{Proof of Upper Bound}


\newproof{pot2}{Proof of Theorem \ref{thm:linear_function_upper_bound}}
\begin{pot2}
 In \cite{witt2013tight}, Witt proves an upper bound on the optimization time of the standard $(1+1)$ EA on any linear function with $n$ relevant bits. We adapt the proof of  \cite{witt2013tight} to obtain an upper bound on the runtime for the scheduled setup.  
In \cite{witt2013tight}, the author defines a potential function $g(x)$ and the random variables $X^t= g(x_t)$, where $x_t$ is the search point at time $t$.
He bounds the multiplicative drift at time $t$ with respect to this potential function, see Equation 4.1 in \cite{witt2013tight}.  For any $\zeta> 1$ (note that Equation 4.1 in \cite{witt2013tight} this variable is called $\alpha$) and any mutation rate $0 <p<1$ it holds:
\begin{align}
\mathbb{E}[X^{t-1}-X^{t}\mid  X^{t-1}=s] &\geq  s p (1-p)^{n-1}\left(1-1/\zeta\right) \ . \label{eq:expected_drift}
\end{align}

More precisely, let $f$ be a linear function depending on the $n$ bits in $I$. Since Algorithm \ref{alg:tdEA} treats each bit symmetrically, we can assume that $f(x)= w_nx_n + \ldots + w_1 x_1$ with $w_n \geq \ldots \geq w_1$. Then, the function $g$ is defined as $g(x)= \sum _{i=1}^n g_ix_i$, where $g_1=1$, $g_i= \min \{\gamma_i, g_{i-1}\frac{w_i}{w_{i-1}}\}$ and $\gamma_i= (1+\frac{\zeta p}{(1-p)^{n-1}})^{i-1}$. 
Let $g$ be the above defined function with $\zeta = \ln \ln n$ and $p=\frac{1}{n}$. In the scheduled setup the $p_t$ changes every round. However, when choosing $\zeta_t= \zeta \frac{p(1-p_t)^{n-1}}{p_t(1-p)^{n-1}}$, then the potential function defined by $p_t$ and $\zeta_t$ coincides with the one defined by $p$ and $\zeta$. Therefore, for every  fixed round $t$, Equation~\eqref{eq:expected_drift} also holds for $p_t$ and $\zeta_t$. If $\zeta_t >1$, then  the expected drift with respect to $g$ at time $t$ is by \eqref{eq:expected_drift} at least

\begin{eqnarray}
\mathbb{E}[X^{t-1}-X^{t} \mid X^{t-1}=s]\geq s p_{t} (1-p_{t})^{n-1}\left(1-1/\zeta_{t}\right) \ . \label{eq:expected_drift_time_t}
\end{eqnarray}

In the sequel, we use this bound on the drift together with standard techniques  to bound $\mathbb{E}[X^T]$ for some $T=(1+o(1))\beta n \ln n$.  Finally, the theorem will follow by applying Markov's inequality.

For $p_t \leq \ln( \zeta^{1/3})p$ we get that 
\begin{align*}
\zeta_t & = \zeta \frac{p}{p_t}\frac{(1-p_t)^{n-1}}{(1-p)^{n-1}}\geq \zeta	\frac{1}{\ln \zeta^{1/3}} \frac{e^{-2p_tn}}{(1-p)^{n-1}}
= \omega(\zeta^{1/4}), 
\end{align*}
where we used Lemma \ref{lemma:e_approx} for the inequality. Define $S = n \ln (n) h( n)$, where $h(n)= (\ln ^{(4)}n)^{-1}$, then $p_S \leq \ln \zeta^{1/3} p$ holds for $n$ large enough. Especially, $\zeta_t >1$ for all $t\geq S$ implies Equation \eqref{eq:expected_drift_time_t} for these $t$.

Note that 
\begin{align*}
\mathbb{E}[X^t]&= \sum_{s} \mathbb{E}[X^t \mid X^{t-1}=s]\Pr(X^{t-1}=s)\\
&\leq  \sum_{s} s \Pr(X^{t-1}=s)\left(1-p_t(1-p_t)^{n-1}\left(1-\frac{1}{\zeta_t}\right)\right)\\
&= \mathbb{E}[X^{t-1}]\left(1-p_t(1-p_t)^{n-1}\left(1-\frac{1}{\zeta^{1/4}}\right)\right)\\
&\leq   \mathbb{E}[X^{S}] \prod_{k=S+1}^t\left(1-p_k(1-p_k)^{n-1}\left(1-\frac{1}{\zeta^{1/4}}\right)\right)\\
&\leq  \mathbb{E}[X^{S}] e^{-\sum_{k=S+1}^t p_k(1-p_k)^{n-1}\left(1-\frac{1}{\zeta^{1/4}}\right)}\\
\end{align*}
As shown in \cite{witt2013tight}, $X^S$ can be bounded:
\begin{align*}
 \mathbb{E}[X^{S}] & \leq \sum_{i=1}^n g_i \leq \sum_{i=1}^n \gamma_i \leq \frac{\left(1+\frac{\zeta p}{(1-p)^{n-1}}\right)^{n-1}-1}{\zeta p (1-p)^{n-1}}\\
  & \leq \frac{e^{n\zeta p(1-p)^{1-n}}}{\zeta p (1-p)^{1-n}}= O\left(\frac{ne^{\zeta / e }}{\zeta }\right) \ .
\end{align*}
Let $T= \beta n \ln n (1+k(n))$, where $k(n)= (\ln^{(4)}n)^{-1}$, and let $\hat{k}(n)=k(n)\alpha e^{-\alpha/ \beta}$. We claim that 
\begin{align}
\sum_{k=S+1}^T p_k(1-p_k)^{n-1}\left(1-\frac{1}{\zeta^{1/4}}\right) \geq \ln (n) (1 + \hat{k}(n) (1+o(1))) \ , \label{eq:bound_sum_p_k}
\end{align} 
which implies
\begin{align*}
\mathbb{E}[X^T] &\leq  e^{ \ln n +\frac{\zeta}{e} - \ln \zeta - \ln (n) ( 1 + \hat{k}(n) (1+ o(1))}
= e^{- \ln (n)  \hat{k}(n) (1+ o(1))} \ . 
\end{align*}
Since $X^T\geq 0$ and $X^T$ cannot  take  values in the interval $(0,1)$, it follows by Markov's inequality that
$\Pr (X^T \neq 0 ) = \Pr (X^T \geq 1) \leq \mathbb{E}[X^T]= o(1)$, which proves the Theorem. 

It is left to show that \eqref{eq:bound_sum_p_k} holds. This can be done by approximating the sum by an integral. Let $A$ be the left hand side of Equation \eqref{eq:bound_sum_p_k} without the $(1-1/\zeta^{1/4})$ term.

Note that 
\begin{align*}
p_t(1-p_t)^{n-1} & \geq   p_t (1-p_t)^n
\geq   p_t e^{-p_tn-p_t^2n} 
\geq   p_t e^{-p_tn}(1-p_t^2n) \text{,}
\end{align*}
where the last two inequalities follow by Lemma \ref{lemma:e_approx}.
Therefore,
\begin{align*}
A &\geq \sum_{t=S}^T p_t e^{-p_tn}(1-p_t^2n) 
\geq \sum_{t=S}^T p_t e^{-p_tn}(1-p_S^2n) \geq \sum_{t=S}^T p_t e^{-p_tn}(1-O(n^{-0.9}))\ , 
\end{align*}
where we used  $p_t \leq p_{S}$ for $t\geq S$ and $p_S^2n= O(n^{-0.9})$.

Next, we approximate the sum by an integral. Note that the function $f(p)= pe^{-np}$ is monotone increasing until $p\leq \frac{1}{n}$ and monotone decreasing for  $p\geq \frac{1}{n}$ and obtains its maximal value $\frac{1}{en}$ at $p=\frac{1}{n}$. Therefore, by Lemma \ref{lemma:int_approx_sum}
\begin{align*}
A &\geq \left( \int_{S}^{T} p_t e^{-n p_t} \mathrm{d}t -\frac{1}{en}\right) (1-O(n^{-0.9})) \text{ .}
\end{align*}
Next, we lower bound the integral. Recall that $S=n \ln (n)h(n)$ and $T= \beta n \ln( n) (1+ k(n))$.
First, we use the variable transformation $x= \frac{t}{n \ln n}$. Then, we use $\frac{n\alpha \ln (x n \ln n)}{x n \ln n}= \frac{\alpha}{x}+ \frac{\alpha \ln(x \ln n)}{x \ln n}\geq \frac{\alpha}{x}$ and Lemma \ref{lemma:e_approx}.
\begin{align*}
 \int_{S}^{T} p_t e^{-n p_t} \mathrm{d}t
 &= \int_{S}^{T} \frac{\alpha \ln t}{t} e^{-n \frac{\alpha \ln t}{t}} \mathrm{d}t\\
&=  \ln n \int_{h(n)}^{\beta(1+k(n))} \frac{n\alpha \ln (x n \ln n)}{x n \ln n} e^{-\frac{n\alpha \ln (x n \ln n)}{x n \ln n}} \mathrm{d}x\\
&\geq   \ln n \int_{h(n)}^{\beta(1+k(n))} \frac{\alpha}{x} e^{-\frac{\alpha}{x}- \frac{\alpha \ln(x \ln n)}{x \ln n}} \mathrm{d}x\\
&\geq   \ln n \int_{h(n)}^{\beta(1+k(n))} \frac{\alpha }{x} e^{-\frac{\alpha }{x}} \left(1-\frac{\alpha \ln(x\ln n)}{x\ln n}\right) \mathrm{d}x\\
&\geq   \ln n \left(1-\frac{\ln(h(n)\ln n)}{h(n)\ln n}\right) \int_{h(n)}^{\beta(1+k(n))} \frac{\alpha }{x} e^{-\frac{\alpha }{x}}  \mathrm{d}x \  ,
\end{align*}
where in the last step we used the monotonicity of the function $\ln x/x$.

The remaining integral evaluates by Definition \ref{def:alpha_beta} to
\begin{align*}
&  1 +\int_{\beta}^{\beta(1+k(n))} \frac{\alpha }{x} e^{-\frac{\alpha }{x}}  \mathrm{d}x- \int_{0}^{h(n)} \frac{\alpha }{x} e^{-\frac{\alpha }{x}}  \mathrm{d}x\\
 &\geq  1 +\int_{\beta}^{\beta(1+k(n))} \frac{\alpha }{\beta(1+k(n))} e^{-\frac{\alpha }{\beta(1+k(n))}}  \mathrm{d}x   - \int_{0}^{h(n)} \frac{\alpha }{h(n)} e^{-\frac{\alpha }{h(n)}}  \mathrm{d}x\\
 &= 1+ k(n)\frac{\alpha }{(1+k(n))} e^{-\frac{\alpha }{\beta(1+k(n))}} - \alpha  e^{-\frac{\alpha }{h(n)}} \\
 &= 1+ k(n)\alpha e^{-\frac{\alpha}{\beta}} (1+o(1)) 
 = 1+ \hat{k}(n) (1+o(1))
\end{align*}
For the inequality we used the monotonicity of the function $\frac{\alpha}{x}e^{-\frac{\alpha}{x}}$ in the regimes $x> \alpha$ and $x < \alpha$.

Since $1/\zeta^{1/4}= o( \hat{k}(n))$, $n^{-0.9}=o(\hat{k}(n))$, $\frac{1}{en}=o(\hat{k}(n))$, and $\frac{\ln(h(n)\ln n)}{h(n)\ln n}=o(\hat{k}(n))$, there exists $n_0$ such that for $n \geq n_0$ Equation \ref{eq:bound_sum_p_k} holds. \qed
\end{pot2}

\section{Adaptive setup}\label{sec:history_dependent_muation_probabilities}

As mentioned in the introduction, Witt's proof in \cite{witt2013tight} can be generalized to obtain the following lower bound for the adaptive setup. 
\begin{theorem}[Lower bound]\label{thm:adaptive_lower_bound}
For any adaptive choice of mutation rates, the runtime of the $(1+1)$ EA on any linear function with $n$ relevant bits is whp at least $(1-o(1))e n\ln n$.
\end{theorem}
Note that this lower bound on the optimization time coincides with the optimization time of the standard $(1+1)$ EA if the number of relevant bits $n$ is known. Since it is known that mutation rate $p=1/n$ is the optimal choice for $n$ relevant bits, it is not surprising that adaptive mutation rates cannot achieve smaller runtime.

Interestingly, we propose an unbiased, comparison-based $(1+1)$ EA (see Algorithm~\ref{alg:adaptiveEA_upper_bound}) with adaptive mutation rate policy that optimizes any linear function with unkown number $n$ of relevant bits in time $(1+o(1))e n\ln n$. The idea of Algorithm \ref{alg:adaptiveEA_upper_bound} is the following. Assume we would know a value $m= \Theta(n)$, say for concreteness $n/2 \leq m \leq 2n$. To estimate the exact value of $n$, we start from a random search point $x$, and create an offspring $y$ with a mutation rate of $p = m^{-1-\eps}$. Note that $p = o(1/n)$, so it is very unlikely to flip more than one relevant bit. Hence, we may assume that no  multi bit flip occurs. Then, the probability to flip a relevant bit is $np$, and with probability $\approx 1/2$ it is a $0$-bit. Hence, if we repeat the test $m$ times, always starting with a new random $x$, then we expect to see $m\cdot np/2$ cases with $f(y) > f(x)$. Let $S$ be the number of observed cases with $f(y) > f(x)$. This number is concentrated, so $ m' := 2S/(mp)$ is a reasonable estimate of $n$. Afterwards, we optimize with the standard (1+1) EA with mutation rate $p':=1/m'$. This approach works if we start with an $m$ such that $n/2 \leq m \leq 2n$. However, if $m$ is too small, then the same test will tell us so, since we will get an estimate $m' > 2m$ in this case. Therefore, Algorithm \ref{alg:adaptiveEA_upper_bound} consists of two parts. In the estimation part, in every iteration, $m$ is doubled and an estimate $m'$ of $n$ is computed as described above. Only if $m/2 \leq m' \leq 2m$,  then the optimization part is executed, that is, the $(1+1)$ EA  is run  with $p=1 / m'$ for $10m' \ln m'$ steps. We show the following theorem.


\begin{theorem}[Upper Bound]\label{thm:adaptive_upper_bound}
The optimization time of Algorithm~\ref{alg:adaptiveEA_upper_bound} on any linear function $f$  with $n$ relevant bits is whp at most $(1+ o(1))en\ln n $.
\end{theorem}
\begin{proof}
Algorithm \ref{alg:adaptiveEA_upper_bound} executes  exponential search (the $m$ is doubled in every round). For each $m$ the estimation part of the algorithm needs $2m$ function evaluations and the optimization part needs $O(m \ln m)$ function evaluations if it is executed and $0$ function evaluations otherwise .  In order to bound the number of function evaluations of Algorithm \ref{alg:adaptiveEA_upper_bound}, we  divide an execution of Algorithm \ref{alg:adaptiveEA_upper_bound} into three phases and use standard concentration inequalities.

We define the first phase by all iterations of the for-loop in line $3$ with $m \leq \sqrt{n}$. Note that there are $O(\ln n)$ iterations of the for-loop since $m$ is doubled in every iteration. We can pessimistically assume that the optimization part is executed in every iteration of the for-loop. Then, the  number of function evaluations in this phase is $O( \sqrt{n}\ln^2 n)$.

The second phase is defined by all iterations of the for-loop in line $3$ with $\sqrt{n}\leq m \leq \frac{n}{100}$. As shown below,  the estimate $m'$ of $n$ will whp never be close to $m$, and therefore, the optimization part  will not be executed. Thus, there are $2m$ function evaluations for each $m$. Since  $m$ grows exponentially, the total number of function evaluations in this phase is $O(n)$. 
It is left to show that whp the optimization part is not executed. Let $\sqrt{n}\leq m \leq \frac{n}{100}$ and $p= \frac{1}{m^{1+\epsilon}}$. 
The variable $S$  counts the number of events $f(y)>f(x)$, where $y$ is the offspring of a randomly initialized bit string $x$. Here, we interpret $S$ as a random variable. 
Consider one mutation step and let $A$ be the event that at least one relevant bit flips. Conditioning on $A$, we claim that $\Pr(f(y)>f(x)) \leq 1/4$. 
Without loss of generality, we can first randomly choose the set $I$ of bits that will be flipped according to the distribution determined by $p$, and then initialize the bit string $x$ randomly. Let $\bar{x}$ be the bit string that differs in every bit from $x$. Let  $y$ and $\bar{y}$ be the mutations of $x$ and $\bar{x}$, respectively, when mutating the bits in $I$. It holds $f(x)-f(y)=-(f(\bar{x})-f(\bar{y}))$. Since  $x$ is chosen uniformly, 
$\Pr\left(f(y)>f(x)\right) =\left(1- \Pr(f(y)=f(x)\right)/2$.
 Now, fix an $i\in I$ and let $\hat{x}$ be the bit string that differs only in position $i$ from $x$. Again denote by $y$ and $\hat{y}$ the mutation of $x$ and $\hat{x}$, respectively, when mutating the bits in $I$. Then, $f(x)=f(y)$ excludes $f(\hat{x})=f(\hat{y})$. Thus, $\Pr(f(y)=f(x))\leq 1/2$ and the claim follows.

 Denote by $q_0$ the probability that $A$ does not happen, that is,  no relevant bit flips. We have argued that $\Pr(f(y)>f(x)) \geq (1-q_0)/4$. It follows that
\begin{align*}
\mathbb{E}[S] &\geq   m \frac{1-q_0}{4}
= \frac{m}{4}(1-(1-p)^n)
 \geq  \frac{m}{4} (1- e^{-np})
\end{align*}
It holds $\mathbb{E}[S] \geq 2 m/ m^{\epsilon}$. To see this, we distinguish two cases.
First, if $np \leq 1/2$, then by Lemma \ref{lemma:e_approx} it holds that $e^{-pn} \leq 1- pn/2$. Therefore, $\mathbb{E}[S] \geq \frac{m}{8}np \geq \frac{m}{8}\frac{100 m}{m^{1+\epsilon}}\geq \frac{2m}{m^{\epsilon}}$.
Second, if $np \geq 1/2$, Then $\mathbb{E}[S] \geq (1-e^{-1/2})/4 \cdot m  \geq 2m/m^{\epsilon}$.  Note that $m'= 2 S m^{\epsilon}$. Since $S$ is the sum of $m$ independent Bernoulli variables,  Lemma \ref{lem:azuma} implies that
\begin{align*}
\Pr(m' \leq 2m ) &= \Pr \left(S \leq \frac{m}{m^{\epsilon}}\right)
\leq \Pr \left(S \leq \mathbb{E}[S] - \frac{m}{m^{\epsilon}}\right)\\
&\leq  e^{-\frac{m^2}{2m^{1+2\epsilon}}}
= e^{- \Omega (n^{0.5-\epsilon})} \ .
\end{align*}

By a union bound argument over $O(\ln n)$ iterations of the for loop in line $3$,  $m' > 2m$ holds whp for all of these iterations in the second phase. 

The third phase is defined by $n/100 \leq m \leq 2 n$. We show that in this phase the estimate $m'$ of $n$ is whp within a $1\pm O(m^{-\epsilon})$ factor of the true value. Therefore, the first time that the optimization part is executed, it holds that $(1- O(m^{-\epsilon}))\frac{1}{n}\leq p\leq (1+ O(m^{-\epsilon}))\frac{1}{n}$. Corollary 4.2 in \cite{witt2013tight} implies that the function $f$ will be optimized after $(1 \pm o(1))e n \ln n $ function evaluations. The number of function evaluations until this happens is $O(n)$, because until then the optimization part was never executed. 
In the sequel we show that the estimate $m'$ for $n$ is  concentrated. Assume that $n/100 < m \leq 2n$. Let $p= m^{-1-\epsilon}$, let $x$ be a random bit string, and let $y$ be an offspring of $x$. 
 Note that the probability of a single bit flip is  $nm^{-1-\epsilon}(1-m^{-1-\epsilon})^{n-1}=nm^{-1-\epsilon}(1-O(m^{-\epsilon}))$, the probability of a multi bit flip bit is $O(m^{-2\epsilon})$, and conditioned on a single bit flip, the probability of $f(y)>f(x)$ is $1/2$. Thus,
\begin{align*}
\Pr(f(y)> f(x))&= \frac{1}{2}\frac{n}{m^{1+\epsilon}}(1+ O(m^{-\epsilon})) 
\end{align*} 
Further, the expected number of times  $f(y)>f(x)$ occurs among $m$ mutations is 
\begin{align*}
\mathbb{E}[S] &= \frac{n}{2 m^{\epsilon}}(1+O(m^{-\epsilon})) \ .
\end{align*}
Since every fitness comparison has effect at most $1$ on $S$, Lemma~\ref{lem:azuma} implies
\begin{align*}
& \Pr((1-m^{-\epsilon})n \leq m' \leq (1+m^{-\epsilon} )n)\\
&= \Pr\left((1-m^{-\epsilon})\frac{n}{2m^{\epsilon}} \leq \frac{m'}{2m^{\epsilon}} \leq (1+m^{-\epsilon} )\frac{n}{2m^{\epsilon}}\right)\\
&= \Pr\left((1-O(m^{-\epsilon}))\mathbb{E}[S] \leq S \leq (1+O(m^{-\epsilon}))\mathbb{E}[S]\right)\\
&\geq  1-2e^{\Omega(m^{1-4 \epsilon})} \ . 
\end{align*} \qed
\end{proof}

\begin{algorithm2e}%
	$m \leftarrow 1$ \;
	$\epsilon \leftarrow 0.01$ \;
	
	\For{$i=1,2, \ldots$}{
		$m \leftarrow 2m$\;
		[ \textsf{Estimate whether $m/2 \leq n \leq 2m$: } ] \\
		$p \leftarrow \frac{1}{m^{1+\epsilon}}$\;
		$S \leftarrow 0$\;
		\For{$j=1, \ldots, m$}{
			$x \in _{u.a.r.} \{0,1\}^{\mathbb{N}}$\;
			$y \leftarrow \textsc{Mutate}(x, p)$\;
			\If{$f(y)> f(x)$}{
				$S \leftarrow S +1 $\;
			}
		}
		$m' \leftarrow 2 S m^{\epsilon}$\;
		\If{$\frac{m}{2}\leq m' \leq 2m$}{
			[ \textsf{Optimize $f$ }: ]\\
			\textit{Run Algorithm \ref{alg:1+1_EA} with $p =\frac{1}{m'}$ for $ 10m'\ln m'$ steps;} \\
		}
	}
	\caption{Adaptive setup: The $(1+1)$ EA with adaptive choice of mutation rates $\vec{p}$  minimizing a pseudo-Boolean function  $f\colon\{0,1\}^{\mathbb{N}}\to\mathbb{R}$ with a finite number $n$ of relevant bits.}
	\label{alg:adaptiveEA_upper_bound}
\end{algorithm2e}

\section{Initial Segment Uncertainty Model  }\label{sec:position_dependent_mutation_probs}
In this section we analyze the runtime of the (1+1) EA with position dependent mutation rates $\vec{p}$ on the \textsc{OneMax} function with support on the initial segment $[n]$. For every summable and monotone decreasing sequence $\vec{p}$ the expected runtime is upper bounded by $O(\ln(n)/ p_n)$, cf. Theorem 14 in  \cite{doerr2015solving}. Note that it is advantageous for this upper bound to take a summable sequence that decays as slowly as possible. However, it is known that there exist no slowest decaying summable sequence (cf. Section 2.6 in \cite{doerr2017unknown}).

Our next theorem states a lower bound on the runtime that is asymptotically as tight as possible. To reduce the technicality of the proof, we will assume that the position dependent mutation rates $\vec{p}$ are monotonically decreasing and smaller than $\frac{1}{2}$, that is, $\frac{1}{2}\geq p_1\geq p_2\geq \ldots $. Further, we define $S_n = S_n(\vec{p}) = \sum_{i=1}^n p_i$ and $S = S(\vec p)=\lim _{n \rightarrow \infty}S_n$.

\begin{theorem}\label{thm:position_dependent_lower_bound}
Let $\vec{p}$ be a monotone decreasing sequence with $p_1 \leq 1/2$, and let $\vec{q}$ be an arbitrary non-summmable sequence. Then there is a constant $c>0$ such that for infinitely many $n \in \N$ the expected optimization time of the (1+1) EA with position dependent mutation probabilities $\vec p$ on the \textsc{OneMax} function with support on the initial segment $[n]$ is at least $c\ln(n)/q_n$.
\end{theorem}

We denote the $k$-fold iterative logarithm by $\ln^{(k)}(x)$, where we truncate any values smaller than $1$ to avoid negative or undefined terms. I.e., we define iteratively $\ln^{(0)}(x) = \max\{1,x\}$ and $\ln^{(k)}(x) = \max\{1, \ln(\ln^{(k-1)}(x))\}$. It is known~\cite[Lemma 2.4]{doerr2017unknown} that the sequence $p_n = 1/ ( n \prod_{j=1}^{\infty} \ln^{(j)}(n))$ is non-summable. Thus, we obtain the following lower bound.
\begin{corollary}
Let $\vec{p}$ be a monotone decreasing sequence with $p_1 \leq 1/2$. Then there is $c>0$ such that the expected optimization time of the (1+1) EA with position dependent mutation probabilities $\vec p$ on the \textsc{OneMax} function with support on the initial segment $[n]$ is at least $c n \ln ^2(n) \prod_{j=2}^{\infty} \ln^{(j)}(n)$ for infinitely many $n\in \N$.
\end{corollary}
Note that this lower bound is tight in the sense that for any $k\geq 0$ the summable sequence $p_n\coloneqq 1/ ( n \prod_{j=1}^{k} \ln^{(j)}(n))$ achieves an optimization time of $O( n \ln ^2 (n)\prod_{j=2}^{k} \ln^{(j)}(n))$ as shown in \cite{doerr2015solving}.\smallskip

\subsection{Proof of Lower Bound}
In the following, we assume that the position dependent mutation rates $\vec{p}$ are monotonically decreasing and smaller than $1/2$. Further, we define $S_n = \sum_{i=1}^n p_i$ and $S=\lim _{n \rightarrow \infty}S_n$.

Before we come to the technical details, we first give an overview over the proof. The crucial step will be to show that the expected runtime is at least $\Omega(\ln(n)M_n)$, where $M_n := \min\{e^{S_n}/(S_np_{\lceil n/2 \rceil}), n^{1.01}/\ln(n)\}$. This will be done in Lemma~\ref{lemma:pos_dep_summable} for the case that $\vec p$ is summable, where the formula can be simplified. The hard part of the proof is to show this bound for non-summable $\vec p$, which is done in Lemma~\ref{lemma:pos_dep_non_summable}. Afterwards, we show by a rather short argument in Lemma~\ref{lemma:bound_runtime_by_non_summable_q} that the inverse of the sequence $M_n$ is summable for every monotone sequence $\vec p$, and that for any non-summable sequence $\vec q$ we have $M_n \geq 1/q_n$ for infinitely many values of $n$.

We start with the lower bound on the optimization time. The first lemma assumes that $\vec p$ is summable, and it follows rather easily from the fact that every bit in $\{ \lceil n/2 \rceil, \ldots, n\}$ that is initialized with $1$ needs to flip at least once. 
\begin{lemma}\label{lemma:pos_dep_summable}
Let $\vec{p}$ be such that $S< \infty$. Then the expected optimization time of the (1+1) EA with position dependent mutation rates $\vec{p}$ on the $\textsc{OneMax}$ function with support on the initial segment $[n]$ is at least $\Omega(\ln(n)/ p_{\lceil n/2 \rceil})$.
\end{lemma}
\begin{proof}
Let $T_n=\frac{\ln n}{4p_{\lceil n/2 \rceil }}$ and let $I \subset \{\lceil n/2 \rceil,  \ldots, n\}$ be the set of bits that are initialized by $1$. By the Chernoff-Hoeffding bounds,  there are whp at least $n/6$ such bits. The probability that bit $i$ is never flipped until time $T_{n}$ is $(1-p_i)^{T_n}\geq e^{-2p_{i}T_{n}}\geq e^{-\ln(n)/2}=n^{-0.5}$. Since all bits are flipped independently, it holds that the probability that all bits in $I$ are flipped until time $T$ is at most $(1- n^{-0.5})^{ n /6}\leq e^{- n^{0.5}/6}=o(1)$. It follows that the expected optimization time is $\Omega(\ln(n)/p_{\lceil n/2 \rceil })$. \qed
\end{proof}

A similar bound holds for non summable sequences, but is much harder to prove. Note that the first term $\frac{\ln (n)e^{S_n}}{S_n p_{\lceil n/2 \rceil}}$ in the bound in Lemma~\ref{lemma:pos_dep_non_summable} generalizes the bound in Lemma~\ref{lemma:pos_dep_summable}, since there we have $S_n = \Theta(1)$.
\begin{lemma}\label{lemma:pos_dep_non_summable}
Let $\vec{p}$ be such that $S= \infty$. The optimization time of the (1+1) EA with position dependent mutation rates $\vec{p}$ on the $\textsc{OneMax}$ function with support on the initial segment $[n]$ is at least $\Omega(\min\{ \frac{\ln (n)e^{S_n}}{S_n p_{\lceil n/2 \rceil}}, n^{1.01}\})$.
\end{lemma}

\begin{proof}
Let us first sketch the proof and give some intuition for the problem.
First, we observe that the bound is easy in some cases: if $p_{n} \leq n^{-1-\delta}$ for $\delta := 0.01$, then it takes a long time to flip the $n$-th bit, and if $S_n \geq 2 \ln n$, then it takes a long time to make the very last step towards the optimum, because we typically flip many bits at once. So we assume that none of these cases happen.
Then, we argue by pigeonhole principle that there is a medium sized set $B$ of bits ($|B| = n^{1-2\delta}$) such that all $p_i$, $i\in B$, differ by at most a factor of $2$, and such that $p_i = O(p_{n/2})$. In particular, it can be shown that $\sum_{i\in B} p_i = o(1)$. We consider the case that $B$ is close to optimal, i.e., that the number $B_1(t)$ of $1$-bits in $B$ is at most $\eps |B|$ for some small $\epsilon$. Then, we study the drift  $\Delta _t := B_1(t)-B_1(t+1)$. The main part of the proof is to show that $\E[\Delta_t] \leq C\epsilon|B|p_{n/2} S_n e^{-S_n}$ for a constant $C$, from which the theorem follows by a lower bound multiplicative drift theorem~\cite{witt2013tight}.

Note that the term $C\epsilon|B|p_{n/2} S_n e^{-S_n}$ roughly resembles the probability that exactly one $1$-bit is flipped in $B$ (probability $\approx \epsilon|B|p_{n/2}$) and at most one bit is flipped in $A\coloneqq [n]\setminus B$ (probability $\approx S_n e^{-S_n}$). However, it would be incorrect to say that this is the leading term of the drift. It is not even necessarily a leading term among those terms that contribute positively to the drift. For example, consider the case that $A$ is not well-optimized, for illustration we may imagine that all of these bits are $1$-bits. Then, a much more likely scenario for an improvement in $B$ is that many bits in $A$ are flipped (which improves the fitness, and has probability $\Theta(1)$ instead of $O(S_n e^{-S_n})$), and one $1$-bit in $B$ is flipped. However, in this case there is an even more likely scenario: a similar combination of bits in $A$ is flipped, a $0$-bit in $B$ is flipped, and thus $B_1(t)$ moves away from the optimum. In general, the situation is more complex than for the case that $A$ consists only of $1$-bits. So what we really show is that all terms that contribute positively to the drift are either at most $C\epsilon|B|p_{n/2} S_n e^{-S_n}$, or they are counterbalanced by even larger terms that contribute negatively to the drift. Nevertheless, this gives a drift bound of $\E[\Delta_t] \leq C\epsilon|B|p_{n/2} S_n e^{-S_n}$, as required.

Let us formalize above ideas. Let $ \delta = 0.01$. If $p_n \leq 	\frac{1}{n^{1+\delta}}$, then the expected  optimization time is at least $\Omega(n^{1+\delta})$, because with probability $\frac{1}{2}$ the $n$-th bit is initialized with $1$ and the expected time until it is flipped for the first time is $\frac{1}{p_n}= n^{1+\delta}$.
 
 If $S_n \geq 2 \ln n$, then the probability that the offspring $y$ of any $x$ is the all $0$ string is 
\begin{align*}
\prod_{x_i=0, \ i \in [n]}(1-p_i)\prod_{x_i=1, \ i \in [n]}p_i
&\leq  \prod_{  i \in [n]}(1-p_i) \leq  e^{-S_n} \ ,
\end{align*}
where the first inequality follows from $p_i\leq 1/2$.
It follows that the expected time until the all $0$-string is hit is at least $e^{S_n}= \Omega(n^{1+\delta})$.

So let us assume that $p_n\geq \frac{1}{n^{1+\delta}}$ and $S_n \leq 2\ln n$. Let $X$ be the expected number of bits  in $[n]$ that flip in one mutation step. The Chernoff-Hoeffding bounds states that $\Pr(X\geq t) \leq 2^{-t}$ for $t\geq 2 e \mathbb{E}[X]$.  Note that $\mathbb{E}[X]=S_n$. Therefore, the Chernoff-Hoeffding bounds imply that $\Pr(X\geq 4 e \ln n) \leq 2^{-4 e \ln n} \leq n^{-2}$. Therefore, a union bound argument over the $n^{1+\delta}$ time steps imply that whp there is no step in which more than $4e\ln n$ bits are flipped. Since we want to prove a lower bound on the expected optimization time, we can assume for the remainder of this proof that at most $4e\ln n$ bits are flipped in every step.

Consider the intervals
$I_k=[ 2^{-k},2^{-(k+1)}]$ for 
$k=1, \ldots  \lceil \log_2 n^{1+\delta} \rceil $. By the pigeon-hole principle there exists a $k$  such that
the set $J_k=\{ i  \mid  n/2 \leq i \leq n, p_i \in I_k \}$ has size at  least 
$\frac{n}{\log_2^3 n}$. Let $\hat{k}$ be the largest such $k$.
Define $B$ to be an arbitrary subset of $ J_{\hat{k}}$ of size $n^{1-2\delta}$, and let $A\coloneqq [n] \setminus B$. Let $p_B=\frac{1}{2^{-(\hat{k}+1)}}$. Note that $2p_B\geq p_i\geq p_B$ for all $i \in B$, and $p_B \leq p_{\lceil n/2 \rceil}$ because $\vec{p}$ is monotone. Since  $|J_k| < \frac{n}{\ln^3_2 n}$ holds for $k> \hat{k}$, there are $(1-o(1))n$ positions $i$ with $p_B \leq p_i$. It follows $(1-o(1))np_B\leq S_n$, which implies $p_B\leq (1+o(1)) 2\ln n /n$. 
Further, for a search point $x^t$ define  $A_0$ and $A_1$ to be the set of positions of $A$ where $x^t$ is $0$ and $1$, respectively. Define $B_0$ and $B_1$ analogously. Next, define
\begin{align*}
S_0^A &\coloneqq \sum_{i \in A, \ x_i^{t}=0} p_i \ , \qquad 
 S_1^A \coloneqq \sum_{i \in A, \ x_i^{t}=1} p_i \ , \qquad
 S^A \coloneqq S_0^A+S_1^A \ ,
\end{align*}
and define $S_0^B$, $S_1^B$, $S^B$ analogously. Note that $S^B \leq  2p_B|B|\leq (1+o(1))2 \ln n / n \cdot n^{1-2\delta} \leq  n^{-\delta}$ for $n$ large enough.

In order to prove the asymptotic lower bound on the optimization time, it turns out that it is enough to consider the time until there are less than $n^{0.5+\delta}$ one bits. Let us assume $n^{0.5+\delta} \leq |B_1|\leq 2n^{0.5+2\delta}$ and define $\epsilon \coloneqq |B_1|/|B|$. In the sequel, we will consider the drift $\Delta _t \coloneqq B_1(t)-B_1(t+1)$, and we will show that there is a constant $C>0$ such that 
\begin{align}\label{eq:drift-of-Delta-t}
\mathbb{E}[\Delta_t \mid |B_1| = \eps |B|] \leq C |B_1| \max\{p_B S_ne^{-S_n},n^{-1-\delta}\}
\end{align}
for all $n^{-0.5+3\delta}\leq \eps \leq 2n^{-0.5+4\delta}$. This will then allow us to lower bound the optimization time with a multiplicative drift theorem. Note that an analogous lower bound on $\mathbb{E}[\Delta_t]$ does not need to hold. Depending on the bits in $A$, the absolute value $|\mathbb{E}[\Delta_t]|$ can be much larger than the right hand side of~\eqref{eq:drift-of-Delta-t}, but only if $\mathbb{E}[\Delta_t] <0$.

So let us show~\eqref{eq:drift-of-Delta-t}. Indeed, this argument will constitute the main part of the proof. Denote by $A_{01}$ the number of $0$-bits in $A$ that flip to $1$ and by $A_{10}$ the number of $1$-bits in $A$ that flip to $0$, and define $B_{01}$ and $B_{10}$ analogously.
In order to bound $\mathbb{E}[\Delta_t]$, we write it as a sum of $6$ terms $D_1, \ldots, D_6$ defined as follows. Note that in the sequel, we always condition on $|B_1|= \epsilon |B|$ but omit this for ease of notation.
\begin{align*}
D_1 & \coloneqq   \Pr(A_{01}< A_{10})\mathbb{E}\left[\Delta_t \mid A_{01} < A_{10} \right]\ ,\\
D_2 & \coloneqq  \Pr(A_{01}=A_{10}=0)\mathbb{E}\left[\Delta_t \mid A_{01}=A_{10}=0 \right]\ , \\
D_3 & \coloneqq  \Pr(A_{01}=A_{10}>0)\mathbb{E}\left[\Delta_t \mid A_{01}=A_{10}>0 \right]\ , \\
D_4 & \coloneqq   \Pr(A_{01}=1+A_{10}=1)\mathbb{E}\left[\Delta_t \mid A_{01}=1+A_{10}=1 \right]\ ,\\
D_5 & \coloneqq   \Pr(A_{01}=1+A_{10}>1)\mathbb{E}\left[\Delta_t \mid A_{01}=1+A_{10}>1 \right]\ ,\\
D_6 & \coloneqq  \Pr(A_{01}\geq 2+A_{10})\mathbb{E}\left[\Delta_t \mid A_{01}\geq 2+A_{10} \right] \ .
\end{align*} 

 The idea is the following. We show that $D_1+D_3+D_5 \leq 0$, while for $D_2, D_4$ and $D_6$ we derive precise upper bounds. Intuitively, it is very unlikely that more than one bit flips in $B$. Therefore, if $A_{01}=A_{10}=0$ or $A_{01}=A_{10}+1=1$ (this corresponds to the terms $D_2$ and $D_4$), then the leading term of $\mathbb{E}[\Delta_t]$  is caused by the event that one $1$-bit and no $0$-bits flip in $B$. If $A_{01}\geq 2+A_{10}$ (this corresponds to $D_6$), then $\mathbb{E}[\Delta_t]$  is small since at least two $1$-bits in $B$ need to be flipped such that the offspring is accepted. Otherwise,  it turns out to be likely that  the number of $1$-bits decreases in $A$ (this corresponds to $D_1$, $D_3$ and $D_5$), which will cause $\mathbb{E}[\Delta_t]$  to be negative because the leading term is caused by the event that one $0$-bit  and no $1$-bit flips in $B$.

 In order to formalize the steps outlined above, define  $F_k$ be the event that $k$ bits flip in $B$.  $\Delta_t=0$ if no bit flips in $B$, and thus, $\mathbb{E}[\Delta_t]= \sum_{k=1}^B \Pr(F_k)\mathbb{E}[\Delta_t\mid  F_k]$. Further, define $\hat{\Delta}_t$ to be the drift $B_1(t)-B_1(t+1)$ assuming that the offspring would  be accepted if and only if $B_1(t)>B_1(t+1)$. More formally, let $\bar{B}_1(t)$ be the number of $1$-bits in $B$ of the offspring at time $t$. Then, $\hat{\Delta}_t= i$ if $B_1(t)-\bar{B}_1(t)=i$ for $i \geq 1$, and $\hat{\Delta}_t =0 $ otherwise. Note that $\Delta_t \leq \hat{\Delta}_t$ holds no matter which events in $A$ we condition on. 
 \begin{claim}\label{claim:F_k}
 It holds $(1-o(1))|B| p_B\leq \Pr(F_1)\leq 2Bp_B$. Further,  it holds that $\Pr(F_k)\leq { B \choose k} (2p_B)^k \leq (2n^{-\delta})^{k}$. Moreover, there is a constant $c_1$ such that  $\mathbb{E}[\hat{\Delta}_t\mid  F_k] \leq c_1\epsilon^2$ for $2 \leq k\leq 4e \ln n$. 
 
 \end{claim}

Let us show this claim.
It holds that $\Pr(F_1)= \sum_{i \in B}\frac{p_i}{1-p_i} \prod_{j \in B} (1-p_j)$ and the first statement of the claim follows by $\prod_{j \in B} (1-p_j)=1-O(S^B) = 1-o(1)$. 

For the second statement note that $\Pr(F_k)= \sum_{I \subset B : |I|=k} \prod_{i \in I}p_i\prod_{i \in B\setminus I }(1-p_i)$ and the first part of the statement follows from $p_i \leq 2 p_B$ and $(1-p_i)\leq 1$. The second part of the statement follows from ${ |B| \choose k} \leq |B|^k$, $|B|= n^{1-2\delta}$ and $p_B\leq (1+o(1))2 \ln n /n $.

Let us prove the third statement. For two $k$-subsets $I$ and $J$ of $B$ denote by $F_I$ and $F_J$ the events that exactly the bits in $I$ and $J$ are flipped, respectively.  $p_i\leq 2p_j$ and $p_i, p_j \leq (1+o(1))4\ln n / n $ implies    $\frac{\Pr(F_I)}{\Pr(F_J)}= \prod _{i \in I} p_i/(1-p_i) \prod_{j \in J} (1-p_j)/p_j\leq 2^k(1+o(1))$. Let us assume that $k$ is odd. 
Let us count the number of $k$-sets $I$ such that $\hat{\Delta}_t= 1+2i$.  The number of such sets $I$ is 
\begin{align*}
{ \epsilon |B| \choose \frac{k+1}{2}+i} {(1-\epsilon)|B| \choose \frac{k-1}{2}-i}
&\leq  \frac{(\epsilon |B|)^{\frac{k+1}{2}+i}}{(\frac{k+1}{2}+i)!}\frac{((1-\epsilon) |B|)^{\frac{k-1}{2}-i}}{(\frac{k-1}{2}-i)!} \ .
\end{align*}

Using Stirlings formula, the total number of $k$ subsets of $B$ is 
\begin{align*}
{ |B| \choose k} &= \frac{|B|!}{k!(|B|-k)!}
= (1+o(1)) \frac{|B|^{|B|}}{k! e^k(|B|-k)^{|B|-k}}\frac{\sqrt{|B|}}{\sqrt{|B|-k}}\\
&= (1+o(1)) \frac{|B|^k}{k! e^k(1-\frac{k}{|B|})^{|B|-k+0.5}}\\
&\geq  (1+o(1))\frac{|B|^k}{k! e^k e^{-k-k(k-0.5)/|B|}}=  (1+o(1))\frac{|B|^k}{k! } \ , 
\end{align*}
where we used Lemma \ref{lemma:e_approx} for the inequality and  $k\leq 4e \ln n $.

Recall that the probabilities $F_I$ and $F_J$ for two $k$-subsets $I$ and $J$ are by at most a factor $2^k(1+o(1))$ apart. In order to obtain an upper bound on $\Pr(\hat{\Delta}_t=2i+1)$, we assume pessimistically that the subsets for which the number of $0$-bits improve are by a factor $2^k(1+o(1))$ times as likely. Therefore, the probability that $\hat{\Delta}_t = 2i +1 $ is at most the quotient of the above two terms multiplied by $2^k(1+o(1))$:
\begin{align*}
\Pr(\hat{\Delta}_t = 2i +1\mid F_k ) &\leq  (1+o(1))2^k \epsilon^{\frac{k+1}{2}+i} (1-\epsilon)^{\frac{k-1}{2}-i}{ k \choose \frac{k+1}{2}+i}\\
&\leq  (1+o(1)) (2^4\epsilon)^{\frac{k+1}{2}+i} \ . 
\end{align*}

It follows easily that there is a constant $c_1$ such that $\mathbb{E}[\hat{\Delta}_t\mid F_k]= \sum_{i=0}^{(k-1)/2} (1+2i) \Pr(\hat{\Delta}_t= 2i+1\mid  F_k)\leq c_1\epsilon^2$ for all odd $k$ in $\{2, \ldots , 4e\ln n\}$.
The calculations for even $k$  are  analogous and Claim \ref{claim:F_k} follows.

\begin{claim}\label{claim:D2}
For $n$ large enough it holds $D_2\leq |B_1|\max \{8e^{-S_n} p_B , n^{-1-\delta}  \} $.
\end{claim}
Denote by $E$ the event $A_{01}=A_{10}=0$.  Note that $\Pr(E)= \prod_{i \in A}(1-p_i)\leq  e^{-S^A}\leq 2e^{-S_n}$.  By Claim \ref{claim:F_k} it holds $\Pr(F_1|E)= \Pr(F_1) \leq 2 B p_B$ and $\mathbb{E}[\Delta_t\mid F_k,E]\leq \mathbb{E}[\hat{\Delta}_t\mid F_k]\leq c_1\epsilon^2$ for $2 \leq k \leq 4e\ln n$. Further, note that $\mathbb{E}[\Delta_t\mid F_1,E]= \epsilon$ and $\sum_{k\geq 2} \Pr(F_k)\leq 1$. It follows that 
\begin{align*}
D_2 \leq 2 e^{-S_n} (2Bp_B\epsilon+ c_1 \epsilon^2) \leq |B_1|\max \{8e^{-S_n} p_B , n^{-1-\delta}  \} \ , 
\end{align*}
where $4e^{-S_n}c_1\epsilon^2/ |B_1| \leq n^{-1-\delta}$ holds for $n$ large enough.


\begin{claim}\label{claim:D4}
For $n$ large enough it holds $D_4\leq |B_1|\max \{8S_ne^{-S_n} p_B , n^{-1-\delta}  \} $.
\end{claim}

Denote by $E$ the event $A_{01}=1+A_{10}=1$. It holds that $\Pr(E)= \sum_{i \in A_0} p_i  \prod_{j \in A\setminus i}(1-p_i)$
which is smaller than $ \sum_{i \in A_0} p_i  e^{-S^A+p_i}\leq 2S_ne^{-S_n}$, where the last inequality follows from $S_B \leq (1+o(1))n^{-\delta}$ and $e^{0.5}< 2$ . 
As above, by Claim \ref{claim:F_k} it holds $\Pr(F_1|E)= \Pr(F_1) \leq 2 B p_B$ and $\mathbb{E}[\Delta_t\mid F_k,E]\leq \mathbb{E}[\hat{\Delta}_t\mid F_k]\leq c_1\epsilon^2$ for $2 \leq k \leq 4e\ln n$. Further, note that $\mathbb{E}[\Delta_t\mid F_1,E]= \epsilon$ and $\sum_{k\geq 2} \Pr(F_k)\leq 1$. It follows that 
\begin{align*}
D_4 \leq 2S_n e^{-S_n} (2Bp_B\epsilon+ c_1 \epsilon^2) \leq |B_1|\max \{8S_ne^{-S_n} p_B , n^{-1-\delta}  \} \ , 
\end{align*}
where $4S_ne^{-S_n}c_1\epsilon^2/ |B_1| \leq n^{-1-\delta}$ holds for $n$ large enough.


\begin{claim}\label{claim:D6}
For $n$ large enough it holds $D_6\leq |B_1|n^{-1-\delta} $.
\end{claim}
Denote by $E$ the event $A_{01}\geq 2+A_{10}$. Clearly, $\Pr(E) \leq 1$. If $A_{01}\geq 2+A_{10}$, then there  need to be at least $2$ bits more that flip in $B_1$ than in $B_0$ such that the number of $0$-bits in $B$ increases. Therefore, at least $2$ bits need to flip in $B$ such that the number of $0$-bits increases. Thus, $\mathbb{E}[\Delta_t|E] \leq \sum_{k\geq 2} \Pr(F_k) \mathbb{E}[\Delta_t|F_k,E] \leq c_1\epsilon^2$, where the last step follows by Claim \ref{claim:F_k} and $\Delta_t \leq \hat{\Delta}_t$. Thus, for $n$ large enough it holds $D_6\leq c_1 \epsilon^2 \leq |B_1|n^{-1-\delta}$. 


\begin{claim}\label{claim:Dodd}
For $n$ large enough it holds $D_1+D_3+D_5\leq 0 $.
\end{claim}
If we condition on  $A_{01}<A_{10} $ and $F_1$, then the offspring will be accepted and therefore $\Delta_t=-1$ happens with probability $1-\epsilon$,  and $\Delta_t=1$ happens with probability $\epsilon$. Thus, $\mathbb{E}[\Delta_t\mid F_1,  \ A_{01}<A_{10}]= -1+2\epsilon$. Since $\Delta_t \leq \hat{\Delta}_t$ and since $\hat{\Delta}_t$ does not depend on the flips in $A$, we can bound $\mathbb{E}[\Delta_t\mid F_k,  \ A_{01}<A_{10}]\leq 
\mathbb{E}[\hat{\Delta}_t\mid F_k]
= O(\epsilon^2)$ for $k\geq 2$. It follows that 
\begin{align} 
\mathbb{E}[\Delta_t\mid  A_{01}<A_{10}] &\leq \Pr(F_1) (-1+2\epsilon) +O(\epsilon^2) \ .  \label{eq1}
\end{align} 
Note that this is smaller than $0$ since by Claim \ref{claim:F_k} it holds $\Pr(F_1)\geq (1+o(1)) B p_B \geq (1+o(1))(n^{-3 \delta})= \omega(\epsilon^2)$.

If $A_{01}=A_{10}>0$, then the number of $0$-bits in $B$ increases if and only if more $1$-bits than $0$-bits flip in $B$ . It holds that  $\mathbb{E}[\Delta_t\mid  F_1, \ A_{01}=A_{10}>0]= \epsilon$ and  $\mathbb{E}[\Delta_t\mid  F_k,\ A_{01}=A_{10}>0]\leq 
\mathbb{E}[\hat{\Delta}_t\mid F_k]=O(\epsilon^2)$ for $k \geq 2$. It follows that 
\begin{align}
\mathbb{E}[\Delta_t\mid A_{01}=A_{10}>0] \leq \Pr(F_1)\epsilon +O(\epsilon^2) \ . \label{eq2}
\end{align}

If $A_{01}=1+A_{10}>1$, then the number of $0$-bits in $B$ increases if and only if more $1$-bits than $0$-bits in $B$ flip. It holds that  $\mathbb{E}[\Delta_t\mid  F_1, \ A_{01}=1+A_{10}>1]= \epsilon$, and  $\mathbb{E}[\Delta_t\mid  F_k, \ A_{01}=1+A_{10}>1]\leq 
\mathbb{E}[\hat{\Delta}_t\mid F_k]=O(\epsilon^2)$ for $k\geq 2$. It follows that 
\begin{align}
\mathbb{E}[\Delta_t\mid  A_{01}=1+A_{10}>1] \leq \Pr(F_1)\epsilon +O(\epsilon^2) \ . \label{eq3}
\end{align}

Recall that
\begin{align*}
 D_1+D_3+D_5
= &\Pr(A_{01}< A_{10})\mathbb{E}[\Delta_t\mid  A_{01}<A_{10}]\\
 &  + \Pr(A_{01}=A_{10}>0)\mathbb{E}[\Delta_t\mid  A_{01}=A_{10}>0]\\
 &  + \Pr(A_{01}=1+A_{10}>1)\mathbb{E}[\Delta_t\mid  A_{01}=1+A_{10}>1] \ .
\end{align*}

Note that the assumption that at most $4e\ln n$ bits are flipped in every step implies (mind the inequality in the first line)
\begin{align}
\Pr(A_{01}< A_{10}) &\geq  \sum_{i=1}^{4e\ln n} \Pr(A_{01}=i-1)\Pr(A_{10}=i) \ , \label{eq4} \\ 
\Pr(A_{01}=A_{10}>0)&=\sum_{i=1}^{4e\ln n} \Pr(A_{01}=i)\Pr(A_{10}=i) \ , \\
\Pr(A_{01}=1+A_{10}>1)&=\sum_{i=1}^{4e\ln n} \Pr(A_{01}=i+1)\Pr(A_{10}=i) \ .
\end{align}

Next, define 
\begin{align*}
\hat{D}_i \coloneqq  &\Pr(A_{01}=i-1)\mathbb{E}[\Delta_t\mid  A_{01}<A_{10}] \\ 
 & + \Pr(A_{01}=i)\mathbb{E}[\Delta_t\mid  A_{01}=A_{10}>0]\\
  & +\Pr(A_{01}=i+1)\mathbb{E}[\Delta_t\mid  A_{01}=1+A_{10}>1] \ .
\end{align*}
In the sequel, we show that $\hat{D}_i\leq 0 $ for all $1\leq i \leq 4e\ln n$. This implies the claim since because  it is fine to underestimate $\Pr(A_{01}< A_{10})$ as done in \ref{eq4} since 
 $\mathbb{E}[\Delta_t\mid  A_{01}<A_{10}]\leq 0$.
$\Pr(A_{01}=k)$ can be lower bounded by
\begin{align*}
 & \frac{\sum_{i\in A_0} p_i }{S_n} \Pr(A_{01}=k)
= \frac{1}{S_n}\sum_{i\in A_0} p_i \sum_{J \subset A_0 : |J|=k }\left(\prod _{j \in J} p_j \prod_{j \in A_0\setminus J}(1-p_j)\right)\\
&\geq  \frac{k+1}{2S_n}  \sum_{J \subset A_0 : |J|=k+1} \left(\prod _{j \in J} p_j \prod_{j \in A_0\setminus J}(1-p_j)\right)
= \frac{k+1}{2S_n}\Pr(A_{01}=k+1) \ ,
\end{align*}
where we used $1-p_j \geq \frac{1}{2}$ and the fact that every factor appears $k+1$ times in the sum. Using $\Pr(A_{01}=k)/\Pr(A_{01}=k+1)\geq (k+1)/(2S_n)\geq 1/(4\ln n)$, we can upper bound $\hat{D}_i/ \Pr(A_{01}=i-1)$ by
\begin{align*}
 & \mathbb{E}[\Delta_t\mid  A_{01}<A_{10}] 
 + 4\ln n\mathbb{E}[\Delta_t\mid  A_{01}=A_{10}>0] \\
 &+ (4\ln n)^2\mathbb{E}[\Delta_t\mid A_{01}=1+A_{10}>1] \\
\leq &  \Pr(F_1)(-1+2\epsilon) + O(\epsilon^2) 
  + 4\ln n\left(\Pr(F_1)\epsilon + O(\epsilon^2)\right)\\ 
& + (4\ln n)^2\left(\Pr(F_1)\epsilon + O(\epsilon^2)\right)\\
\leq  & 0 \ , 
\end{align*}
 where  \eqref{eq1}, \eqref{eq2} and \eqref{eq3} imply the first inequality, and the second inequality follows for $n$ large enough since $\Pr(F_1) \geq (1-o(1))|B|p_B \geq n^{-3\delta}$ and $\epsilon \leq 2n^{-0.5+2\delta}$. This proves Claim~\ref{claim:Dodd}. 
 
 Altogether, the claims~\ref{claim:D2},~\ref{claim:D4},~\ref{claim:D6}, and~\ref{claim:Dodd} prove ~\eqref{eq:drift-of-Delta-t}.

In the following, we apply the lower bound multiplicative drift theorem, Theorem~\ref{thm:multiplicative-drift}, on the process $B_1(t)$. Since we assume that in every step at most $4e\ln n$ bits are flipped,  there is a point in time with $B_1(t)= c n^{0.5 +2\delta}$ for some $1\leq c \leq 2$ or the optimum will not be found. Define $\kappa = 4e\ln n$, then condition~(1) of Theorem \ref{thm:multiplicative-drift} is satisfied. Further, let $s_{\min}=n^{0.5+\delta}$ and $s_0=  cn^{ 0.5+2\delta}$.  
Let $T$ be the first point in time such that $B_1(t) \leq s_{\min}$. Let $\hat{\delta}= C\max\{p_BS_ne^{-S_n}, n^{-1-\delta}\}$. By Equation \ref{eq:drift-of-Delta-t} it holds $\mathbb{E}[\Delta_t\mid  B_1(t)= s] \leq\hat{\delta}\cdot s$ for $n^{0.5+\delta}\leq s \leq 2n^{0.5+2\delta}$, which is Condition~(2) in Theorem~\ref{thm:multiplicative-drift}.
\footnote{Here we cheat slightly, as we only show this condition for $s_{\min} \leq s \leq 2n^{0.5+2 \delta }$, not for all $s \geq s_{\min}$. However, the proof of Theorem~\ref{thm:multiplicative-drift} in~\cite{doerr2017bounding} reduces multiplicative drift to additive drift by considering the rescaled random variable $Y_t := \ln(B_1(t))$. It was shown in~\cite{doerr2017bounding} that this has still at most constant additive drift whenever $s_{\min} \leq B_1(t) \leq s_0$, and the additive drift theorem only requires drift for these values.}
Therefore, the multiplicative drift theorem implies that 
\begin{align*}
&\mathbb{E}[T\mid B_1(0)= s_0]\geq \frac{1+\ln s_0 - \ln s_{\min}}{2\hat{\delta}+ \frac{\kappa^2}{s_{\min}^2-\kappa^2}} \\
&\geq \frac{1+\ln s_0 - \ln s_{\min}}{\max \{4 p_BS_ne^{-S_n}, 4 n^{-1-\delta}, 2\frac{\kappa^2}{s_{\min}^2-\kappa^2}\}}
\geq c_2 \ln n \min \{\frac{e^{S_n}}{S_np_B}, n^{1+\delta}\} \ ,
\end{align*}
where the last inequality  holds for $n$ and $c_2$ large enough.
 This concludes the proof since $p_B \leq p_{\lceil n/2 \rceil}$.

 \qed
\end{proof}

The next lemma links the bound on the runtime with non-summable sequences.

 \begin{lemma}\label{lemma:bound_runtime_by_non_summable_q}
Let $(p_n)_{n\geq 1}$ and $(q_n)_{n\geq 1}$ be sequences of positive reals with $p_n \leq 1$ for all $n \in \N$. Let $I \subseteq \N$ be the set of all indices $n$ such that
\begin{align}\label{eq:discrete_variable_transform}
\min\left\{\frac{n^{1.01}}{\ln n},\frac{\exp\{\sum_{i = 1}^n p_i\}}{p_{\lceil n/2 \rceil} \cdot \sum_{i = 1}^n p_i} \right\}\leq \frac{1}{q_n}.
\end{align}
Then $\sum_{n\in I}q_n <\infty$. In particular, if $\sum_{n=1}^{\infty}q_n = \infty$  then $|\N \setminus I| = \infty$.
\end{lemma}
\begin{proof}
We abbreviate $a_n := \ln n/n^{1.01}$ and $b_n := p_{\lceil n/2 \rceil} \sum_{i = 1}^n p_i \cdot \exp\{-\sum_{i = 1}^n p_i\}$, so the left hand side of~\eqref{eq:discrete_variable_transform} is $\min\{a_n^{-1},b_n^{-1}\} = (\max\{a_n,b_n\})^{-1}$. By taking the inverse of~\eqref{eq:discrete_variable_transform} and summing over all $n \in I$, we get
\begin{align}\label{eq:discrete_variable_transform2}
\sum_{n\in I} q_n \leq \sum_{n\in I}\max\{a_n,b_n\} \leq \underbrace{\sum_{n=1}^{\infty} a_n}_{=:S_a} + \underbrace{\sum_{n=1}^{\infty} b_n}_{=:S_b}.
\end{align}
Obviously $S_a < \infty$, so it remains to show $S_b<\infty$. To ease notation, let $f(x) := xe^{-x}$, so $S_b = \sum_{n=1}^{\infty}p_{\lceil n/2 \rceil} f(\sum_{i=1}^n p_i)$. The function $f$ is easily seen to be increasing from $0$ to $1$, and to be decreasing afterwards.

If $\sum_{n=1}^{\infty} p_n = c < \infty$, then $S_b \leq \sum_{n=1}^{\infty}p_{\lceil n/2 \rceil} c \leq 2 c^2$, and we are done. So assume $\sum_{n=1}^{\infty} p_n = \infty$. Then for all $k\in \N$ the index $n_k := \min\{n \geq 1 \mid \sum_{i = 1}^n p_i \geq k\}$ is well-defined, and we set $n_0 := 1$. Since $p_n \leq 1$, we have
\[
\sum_{n= 2n_k-1}^{2n_{k+1}-2} p_{\lceil n/2 \rceil} =2\bigg(p_{n_k} +\sum_{i = 1}^{n_{k+1}-1} p_i - \sum_{i= 1}^{n_{k}} p_i\bigg) \leq 4.
\]
Therefore,
\begin{align*}
& S_b  = \sum_{k=1}^{\infty}\sum_{n=2n_{k-1}-1}^{2n_k -2} p_{\lceil n/2 \rceil} f\bigg(\sum_{i = 1}^n p_i\bigg) \leq \sum_{k=1}^{\infty}\sum_{n=2n_{k-1}-1}^{2n_k -2} p_{\lceil n/2 \rceil} \sup_{x \in [k-1,\infty)}\{f(x)\} \\
& \leq \sum_{k=1}^{\infty} 4\cdot  \sup_{x \in [k-1,\infty)}\{f(x)\} = 4\bigg(f(1) + \sum_{k=2}^{\infty}f(k-1)\bigg)  = \frac{4}{e} + \frac{4e}{(e-1)^2} < \infty,
\end{align*}
as required. \qed

\end{proof}

\newproof{pot7}{Proof of Theorem \ref{thm:position_dependent_lower_bound}}

\begin{pot7}
Let $\vec{p}$ be any monotone decreasing sequence with $p_1\leq 1/2$, and let $\vec{q}$ be any non-summable sequence. If $\vec{p}$ is summable, then Lemma \ref{lemma:pos_dep_summable}, otherwise  Lemma \ref{lemma:pos_dep_non_summable}, implies that the expected optimization time is $\Omega(\min\{ \frac{\ln (n)e^{S_n}}{S_n p_{\lceil n/2 \rceil}}, n^{1.01}\})$. Then, Lemma \ref{lemma:bound_runtime_by_non_summable_q} implies the theorem. \qed
\end{pot7}

\section{Conclusions}
\label{sec:conclusions}
We have precisely analyzed the optimal strategies for the hidden subset problem  for the scheduled and the adaptive setup. Both are asymptotically faster than the best strategy for the static setup. For the adaptive setup, the unknown $n$ does not increase the runtime. For the non-adaptive setup, there is a price to pay, namely we lose a factor of $\beta/e \approx 1.307$ in the runtime. The best algorithm in this case follows a rather natural schedule $p_t = \alpha t/\ln t$, except for the surprising factor $\alpha$. The best schedule is surprisingly rigorously determined, and even slight deviations from the optimal schedule lead to a loss in performance. On the other hand, the algorithm that achieves runtime $(1\pm o(1))en\ln n$ in the adaptive case is arguably rather artificial and ad hoc. Most common strategies like the $1/5$-rule adapt the mutation rate in small steps, see~\cite{eiben1999parameter,karafotias2015parameter} for reviews. It is an interesting question whether the same runtime can be achieved with such strategies.

Another intriguing question is on the connection between the hidden subset problem and the initial segment uncertainty model. On all studied fitness functions, the optimal runtimes of these algorithms are asymptotically equal -- for \leadingones the connection is even more intimate. It remains an open question whether a general connection can be found between the two models.

\subsection*{Acknowledgments}
Marcelo Matheus Gauy was supported by CNPq grant no. 248952/2013-7.  Asier Mujika was supported by the Swiss National Science Foundation CRSII5-173721.


\begin{footnotesize}
\bibliographystyle{elsarticle-num}
\bibliography{ref}
\end{footnotesize}

\end{document}